\documentclass[journal,letterpaper]{IEEEtran}
\usepackage{amsmath}
\usepackage{amssymb}
\usepackage{booktabs}
\usepackage{graphicx}
\usepackage{subfigure}
\usepackage{float}
\usepackage{url}
\usepackage{multirow}
\usepackage{amsthm}
\usepackage{threeparttable}
\usepackage[linesnumbered,boxed,ruled,commentsnumbered,noend]{algorithm2e}
\usepackage{color}

\newtheorem{lemma}{Lemma}
\newtheorem{theorem}{Theorem}

\theoremstyle{definition}

\ifCLASSINFOpdf

\else

\fi

\usepackage{xcolor}

\begin{document}
%
\title{Multi-Robot Trajectory Planning with Feasibility Guarantee and Deadlock Resolution: An Obstacle-Dense Environment }
%
%
%

\author{Yuda Chen, Chenghan Wang, Meng Guo and  Zhongkui Li,~\IEEEmembership{Senior Member,~IEEE}
\thanks{
	The authors are with the State Key Laboratory for Turbulence and
	Complex Systems, Department of Mechanics and Engineering Science,
	College of Engineering, Peking University, Beijing 100871, China.
        Corresponding author: \texttt{zhongkli@pku.edu.cn}.
}
}

%
%

\markboth{Journal of \LaTeX\ Class Files,~Vol.~14, No.~8, August~2015}%
{Shell \MakeLowercase{\textit{et al.}}: Bare Demo of IEEEtran.cls for IEEE Journals}

\maketitle

\pagestyle{empty}  
\thispagestyle{empty} 

\begin{abstract}
	This article presents a multi-robot trajectory planning method which not only guarantees optimization feasibility and but also resolves deadlocks in obstacle-dense environments.
	The method is proposed via formulating a recursive optimization problem, where a novel safe corridor is generated online to ensure obstacle avoidance in trajectory planning.
	A dynamic-priority mechanism is combined with the right-hand rule to handle potential deadlocks that are much harder to resolve due to static obstacles.
	Comparisons with other state-of-the-art results are conducted to validate the improved safety and success rate.
	Additional hardware experiments are carried out with up to eight nano-quadrotors in various cluttered scenarios.
\end{abstract}


%
\IEEEpeerreviewmaketitle


\section{Introduction}

\IEEEPARstart{C}{ollision}-free trajectory planning plays an essential role for a swarm of robots navigating in a shared environment~\cite{Chung2018}.
The robots need to avoid collision with both other robots and obstacles while moving to their targets.
Various techniques are proposed, among which optimization-based methods~\cite{Luis2019,Jesus2021} have gained significant attentions due to their flexibility of adding constraints, such as the convex constraints in~\cite{Luis2019,Augugliaro2012} and the non-convex constraints in~\cite{Jesus2021}.
Nevertheless, the constrained optimization may suffer from the infeasibility, leading to failed planning.
This problem becomes specially severe in a crowed environment with dense obstacles.
Moreover, without a central coordinator, it often happens that the robots block each other indefinitely  and no further progresses can be made, which is also known as deadlocks~\cite{Alonso2018}.

To address the issue of infeasibility, the work in~\cite{Zhou2021} converts hard constraints into soft ones by adding them to the objective function.
But the safety properties associated with the constraints such as collision avoidance cannot be ensured anymore.
Our previous work in~\cite{Chen2022} can provably ensure the feasibility of underlying optimization, which however cannot deal with the obstacle collision.
To handle obstacles during motion, one of the most widely-adopted method is safe corridor, e.g., \cite{Deits2015}~generates a high-quality corridor via semi-definite programming, and~\cite{Liu2017} combines the search-based path planning and geometry-based corridor construction.
Unfortunately, the hard constraints introduced by the safe corridor may also lead to infeasible optimizations, as discussed in \cite{Toumieh2022,Senbaslar2021}.
In addition, the authors in \cite{Park2020,Park2022} propose a cubic corridor in a grid map.
Although this method is computationally efficient  to generate safe corridor and can guarantee feasibility, it is restricted by the cubic shape that may have a lower space efficiency.
Furthermore, the method in~\cite{Honig2018} generates a safe corridor via supported vector machines, which has a higher utility rate of freespace, but is centralized and computed offline.
Another way to construct a safe corridor is the voxel expansion proposed in~\cite{Toumieh2022,Gao2020} on a grid map,
which however is also computed offline.

\begin{figure}[t!]
	\centering
	\includegraphics[width=0.85\linewidth]{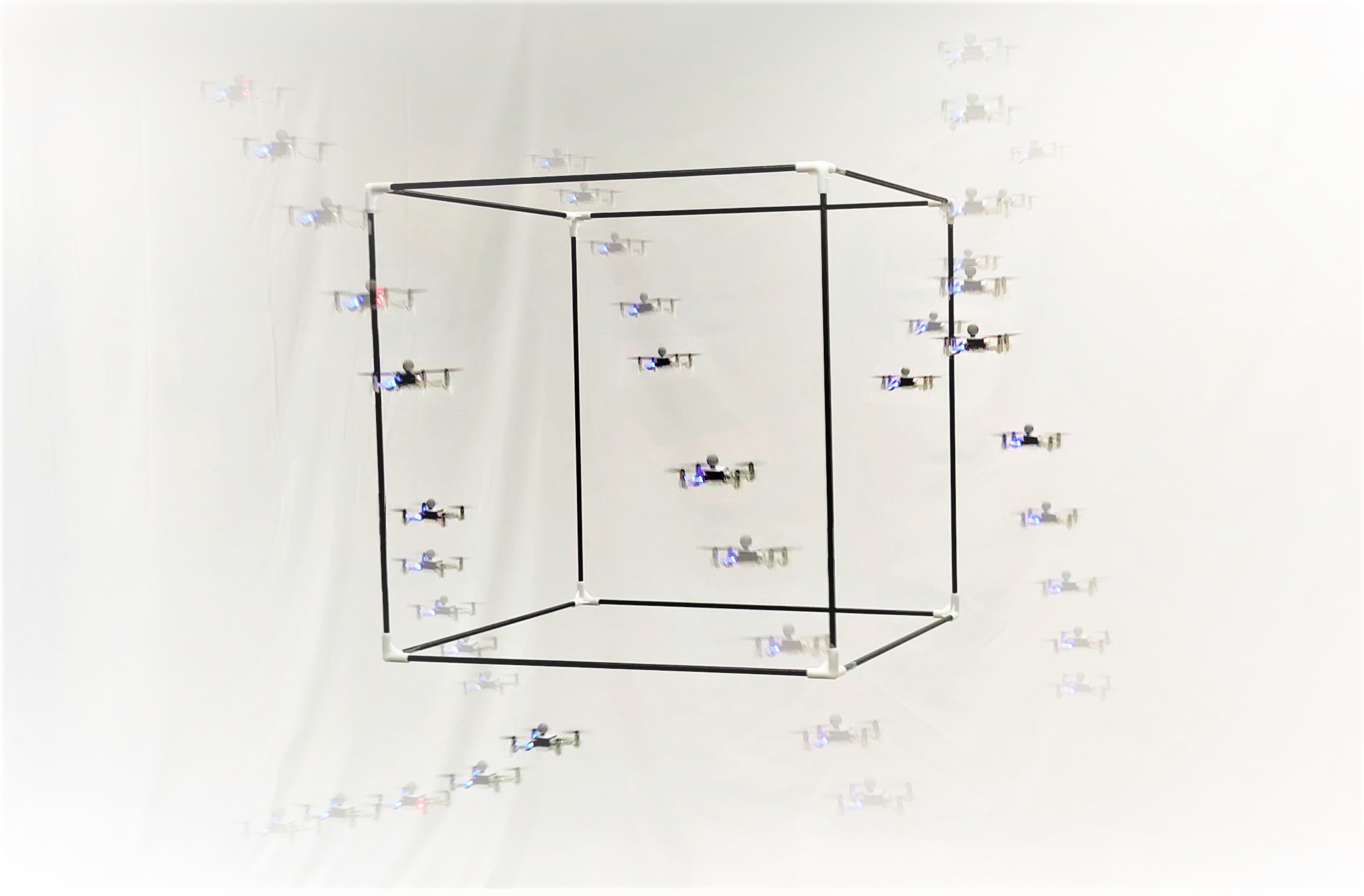}
	\caption{Eight nano-quadrotors fly through a cubic framework.}
	\label{8-cube-framework}
\end{figure}

As mentioned earlier, another prominent problem in distributed trajectory planning is deadlock when robots block each other indefinitely during motion, as also addressed in~\cite{Park2022,Wang2017}.
Various heuristic approaches are proposed to resolve deadlocks, by introducing e.g., the right-hand rule~\cite{Toumieh2022,Zhou2017}, priority planning~\cite{Park2022}, and detour points~\cite{Abdullhak2021}.
The common drawback of these approaches is that the safety of robots or the effectiveness of deadlock resolution is not ensured and usually can lead to livelock where deadlocks re-appear after resolution.
Our previous work~\cite{Chen2022} proposes an adaptive right-hand rule for deadlock resolution in obstacle-free space.
Unfortunately, it cannot be applied to obstacle-dense environments because static obstacles are non-cooperative and cannot be pushed away by the right-hand force.

To address these two issues formally in obstacle-dense environments, we first propose a novel method to construct safe corridors online which provide feasible space for trajectory generation.
It adopts the path planning method ABIT$^\star$ \cite{Strub2020} to determine a reference path, based on which the separating hyperplanes are computed via quadratic programming between the imminent obstacles and the planned trajectory.
Regarding deadlocks in obstacle-dense environments, a dynamic-priority mechanism is proposed along with the adaptive right-hand rule.
Invalidation of the right-hand rule would induce a higher priority thus more egoistic to continue its motion.

The main contributions are summarized as follows.

\begin{itemize}
\item[$\bullet$]
The proposed online generation of safe corridor as a sequence of polytopes has a much higher utilization of the workspace than the cubic safe corridor proposed in~\cite{Park2020,Park2022}.
It is shown to reduce computational cost and increase safety margin, compared with~\cite{Deits2015}.

\item[$\bullet$]
The proposed dynamic-priority mechanism resolves potential deadlocks whenever the adaptive right-hand rule \cite{Chen2022} is invalidated due to static obstacles.

\item[$\bullet$] Comparisons with several state-of-the-art methods \cite{Jesus2021,Zhou2021,Park2022} are made in non-trivial cluttered scenarios.
  It is validated that the proposed method has a better performance in terms of guaranteeing collision avoidance as well as handling deadlocks.
\item[$\bullet$] Extensive hardware experiments are conducted to validate the real-time performance, including eight nano-quadrotors navigating in obstacle-dense 3D scenarios,
  as shown in Fig.~\ref{8-cube-framework}.

\end{itemize}

\section{Problem Statement and Preliminaries}

Consider a team of~$N$ robots that navigate to their respective destinations within a common 2D or 3D workspace that are cluttered with static obstacles.
During the navigation, a robot cannot collide with either other robots or obstacles.
Each robot can determine its own control input and exchange information with other robots via communication.

\subsection{Robot Dynamics}
Let $h>0$ denote the sampling time.
The dynamic model of the robot~$i\in \mathcal{N}=\{1,\cdots,N\}$ is given by
\begin{equation} \label{eq:dynamic-constraint}
	x_{k}^{i}(t)=\mathbf{A} x_{k-1}^{i}(t)+\mathbf{B} u_{k-1}^{i}(t),
\end{equation}
where~$k\in \mathcal{K} \triangleq \{1,\cdots,K\}$ is the step within the planning horizon~$K>0$;
$p^i_k(t),\, v^i_k(t),\, u^i_k(t)$ are the \emph{planned} position, velocity, and control input at time $t + k h$, respectively;
$x^i_k(t)=[p^i_k(t),\, v^i_k(t)]$ is the planned state at time $t+kh$ with $x^i_0(t)=x^i(t)$;
the double-integrator dynamics are given by
$\mathbf{A}
=
\left[
\begin{array}{ccc}
	\mathbf{I}_d & h \mathbf{I}_d  \\
	\mathbf{0}_d & \mathbf{I}_d \\
\end{array}
\right]$,
$\mathbf{B}
=
\left[
\begin{array}{ccc}
	\frac{h^2}{2} \mathbf{I}_d   \\
	h \mathbf{I}_d
\end{array}
\right]$
with the dimension $d=2,3$.
The planned trajectory of robot $i$ at the time~$t>0$ is defined as
$\mathcal{P}^i(t)=\left\{ p^i_1(t),p^i_2(t),\cdots,p^i_K(t) \right\} $.
Additionally, the velocity and input constraints are given as
\begin{equation} \label{eq:input-state-constraints}
  \begin{split}
	\| \Theta_a u^i_{k-1}(t) \|_2 &\le a_{\text{max}} ,\ k \in \mathcal{K},\\
	\| \Theta_v v_k^i \|_2 &\le v_{\text{max}}, \  k \in \mathcal{K},
\end{split}
\end{equation}
where $\Theta_v,\, \Theta_a$ are given positive-definite matrices, and $v_{\text{max}},\, a_{\text{max}}$ denote the maximum velocity and acceleration, respectively.

Once the planned trajectory is computed, a lower feedback controller is followed to
drive robot~$i$ within the time interval $\left[ t,\, t+h \right]$,
such that $x^i(t+h)=x^i_1(t)$ holds when robot~$i$ re-plans at time~$t+h$.
Furthermore, at each time step~$t$, robot~$i$ sends its predetermined trajectory~$\overline{\mathcal{P}}^i(t)$ to other robot~$j \neq i$, which is constructed directly from the planned trajectory~$\mathcal{P}^i(t-h)$ derived at~$t-h$.
More specifically, $\overline{\mathcal{P}}^i(t)=\left\{ \overline{p}_{1}^{i}(t),\, \overline{p}_{2}^{i}(t),\, \cdots,\, \overline{p}_{K}^{i}(t) \right\} $, where
$\overline{p}_{k}^{i}(t) = p^i_{k+1}(t-h)$, $\forall k \in \tilde{\mathcal{K}} \triangleq \{1,2,\cdots,K-1\}$ and $\overline{p}^i_{K}(t) \triangleq p^i_{K}(t-h)$.
For the sake of simplicity, the time index~$t$ is omitted in the sequel, whenever ambiguity is not caused.

\subsection{Collision Avoidance}

\subsubsection{Inter-Robot Collision Avoidance}
The safety area around each robot~$i\in \mathcal{N}$ is represented by a ball as $\mathcal{R}^i = \left\{ x + p^i \ | \ \|x\|_2 \leq r_a \right\}$, where $r_a>0$ is a given safety radius.
To avoid inter-robot collisions, any pair of robots should satisfy $ \| p^i - p^j \|_2 \geq 2 r_a$.
This constraint can be converted equivalently to linear constraints via
the modified buffered Voronoi cells with warning band (MBVC-WB) proposed in our previous work~\cite{Chen2022} as follows:
\begin{subequations} \label{eq:inter-constraint}
	\begin{align}
		&{a_{k}^{i j}}^\mathrm{T} p_{k}^{i} \geq b_{k}^{i j}, \; \forall j\neq i,\, \forall k \in \tilde{\mathcal{K}},  \label{eq:a p > b 1} \\
		&{a_{K}^{i j}}^\mathrm{T} p_{K}^{i} \geq b_{K}^{i j} + w^{i j},\; \forall j \neq i \label{eq:a p > b 2}.
	\end{align}
\end{subequations}
where
$a_{k}^{i j}=\frac{ \overline{p}_{k}^{i}-\overline{p}_{k}^{j} } { \|\overline{p}_{k}^{i}-\overline{p}_{k}^{j}\|_{2} }$,
$b_{k}^{i j}={a_{k}^{i j}}^\mathrm{T} \frac{\overline{p}_{k}^{i} + \overline{p}_{k}^{j}}{2}+\frac{r_{\text{min}}}{2}$,
$r_{\text{min}} = \sqrt{ 4 { r_a }^2+h^{2} v_{\text{max} }^{2}}$ denotes the extended minimum distance, and $w^{i j}\in [0,\, \epsilon]$ is an additional variable as the warning band with the maximum
width~$\epsilon>0$.

\subsubsection{Obstacle Avoidance}
Let $\mathcal{O}\subset \mathbb{R}^d$ denote the set of obstacles' occupied space.
Similar to~\cite{Senbaslar2021}, the obstacles are assumed to be convex-shaped.
Thus, the constraint of obstacle avoidance requires that the safety area of each robot does not intersect with any obstacle, i.e., $\mathcal{R}^i \cap \mathcal{O} = \emptyset$, $\forall i\in \mathcal{N}$.
Namely, a planned trajectory $\mathcal{P}$ is collision-free, if
\begin{equation} \label{eq:oringinal-obstacle-constraint}
	\textbf{Conv} (\left\{ p^i_k, p^i_{k+1} \right\}) \cap  \tilde{\mathcal{O}} = \emptyset, \, \forall k \in \tilde{ \mathcal{K} },
\end{equation}
holds, where $\tilde{\mathcal{O}}$ is the set of occupied space after inflating obstacles by $r_\text{a}$ and
$\textbf{Conv} (P) \triangleq \left\{ \sum_{i=1}^n \theta_i p_i \ | \ p_i \in P, \ \theta_i \geq 0, \ \sum_i \theta_i =1, \ i=1,2,\cdots,n \right\}$ is the convex hull formed by points in the set $P$.

\subsection{Problem Statement}
Assuming that all robots are collision-free initially at time~$t_0$,
our goal is to design a distributed trajectory planning algorithm that drives each robot to its target~$p^i_{\text{target}}$, $\forall i \in \mathcal{N}$,
while respecting the constraints specified in~\eqref{eq:dynamic-constraint}-\eqref{eq:oringinal-obstacle-constraint}.

\section{Trajectory Planning Method}

The overall trajectory planning algorithm will be described in this section,
which consists of the construction of safe corridor, the deadlock resolution scheme, the trajectory optimization algorithm and the proof of feasibility guarantee.

\subsection{Safe Corridor Construction} \label{subsection:safe-corrdior}

The obstacle avoidance is realized by restricting the planned trajectory to be inside a safe corridor.
The corridor is constituted by a sequence of convex polytopes of which the boundaries separate the planned positions $p^i_k$, $i \in \mathcal{N}$, $k \in \mathcal{K}$ and inflated obstacles.
Based on the safe corridor, the obstacle avoidance constraint can be written as
\begin{equation} \label{eq:obstacle-constraint}
	{a_{k}^{i,o}}^\mathrm{T} p_{k}^{i} \geq b_{k}^{i,o}, k \in \mathcal{K},
\end{equation}
where $a_{k}^{i,o}$ and $b_{k}^{i,o}$ determines the boundary of the corridor for robot~$i$ at horizon $k \in \mathcal{K}$ as $\mathcal{B}^i_k = \{ p \ | \ {a_{k}^{i,o}}^\mathrm{T} p = b_{k}^{i,o} \}$.
As shown in Fig.~\ref{corridor}, the safe corridor is constructed in three main steps: path planning, segmentation of EPT and computation of the separating hyperplanes.

\subsubsection{Path Planning}

To begin with, a reference path $\Gamma^i = \{ \overline{p}^i_K, \cdots, p^i_{\text{target}} \}$ is generated for each robot~$i$, which connects the terminal horizon position of~$\overline{p}^i_K$ and its target $p^i_{\text{target}}$ as shown in Fig~\ref{get_path}.
In this work, the algorithm of Advanced Batch Informed Trees (ABIT$^\star$) proposed in~\cite{Strub2020} is adopted to plan this reference path.
Once the path $\Gamma^i$ is obtained, a tractive point denoted by~$p^i_{\text{tractive}}$ is found as $p^i_{\text{tractive}} = \Gamma^i_{m^\star} $, where~$m^\star$ is the maximum index~$m$ such that
\begin{equation} \label{get-tractive-point}
\textbf{Conv}  (\left\{ \Gamma^i_m, \overline{p}^i_K \right\}) \cap  \tilde{\mathcal{O}} = \emptyset,
\end{equation}
where $\Gamma^i_m$ is the $m$-th point of $\Gamma^i$.
An illustration is shown in Fig.~\ref{get_segment_list}.
Then, the extended predetermined trajectory (EPT) is defined as $\tilde{\mathcal{P}}^i = \left\{ \overline{p}_{1}^{i},\cdots,\overline{p}_{K+1}^{i} \right\} $  where $\overline{p}_{K+1}^{i} = p^i_{\text{tractive}}$.
Notably, it is unnecessary to update the path at every step as tractive point~$\Gamma^i$ can be determined via~\eqref{get-tractive-point} given the previous path.
In other words, the robot only needs to replay its path if its target is not changed.

\subsubsection{Segmentation of EPT}
In order to reduce the computational complexity, the derived EPT~$\tilde{\mathcal{P}}^i$ is divided into $N_s$ segments, denoted by~$\mathcal{S}^i_n$, $n = 1,2,\cdots, N_s$, as illustrated in Fig~\ref{get_segment_list}.
These segments are computed in three main steps:
First, $\mathcal{S}^i_1$ is initialized as $\mathcal{S}^i_1 = \left\{ \overline{p}_{K+1}^{i} \right\}$.
Then, the point~$\overline{p}_{k}^{i}$ is added into the current segment $\mathcal{S}^i_1$ consecutively
from $k=K$ to $k=1$, until the convex hull formed by~$\mathcal{S}^i_1$ is not collision-free anymore, i.e., $\textbf{Conv} (\mathcal{S}^i_1) \cap \tilde{\mathcal{O}} \neq \emptyset$.
Afterwards, the next segment $\mathcal{S}^i_2$ is initialized as $\mathcal{S}^i_2 = \left\{ \overline{p}_{k}^{i} \right\}$ where $\overline{p}_{k}^{i}$ is the last point of $\mathcal{S}^i_1$.
The above process is repeated until the starting point~$\overline{p}^i_1$ is added into the last segment $\mathcal{S}^i_{N_s}$.

\begin{figure}[t!]
	\centering
	\subfigure[Path planning.]{
		\includegraphics[width=0.45\linewidth]{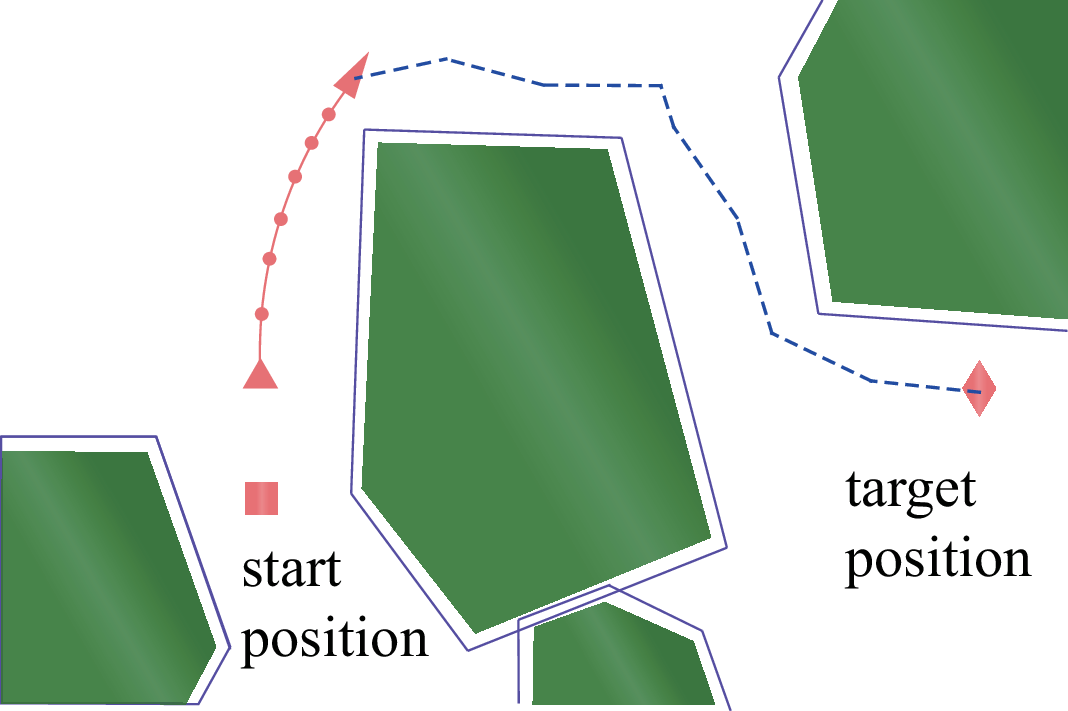}
		\label{get_path}
	}
	\subfigure[Segmentation of EPT.]{
		\includegraphics[width=0.45\linewidth]{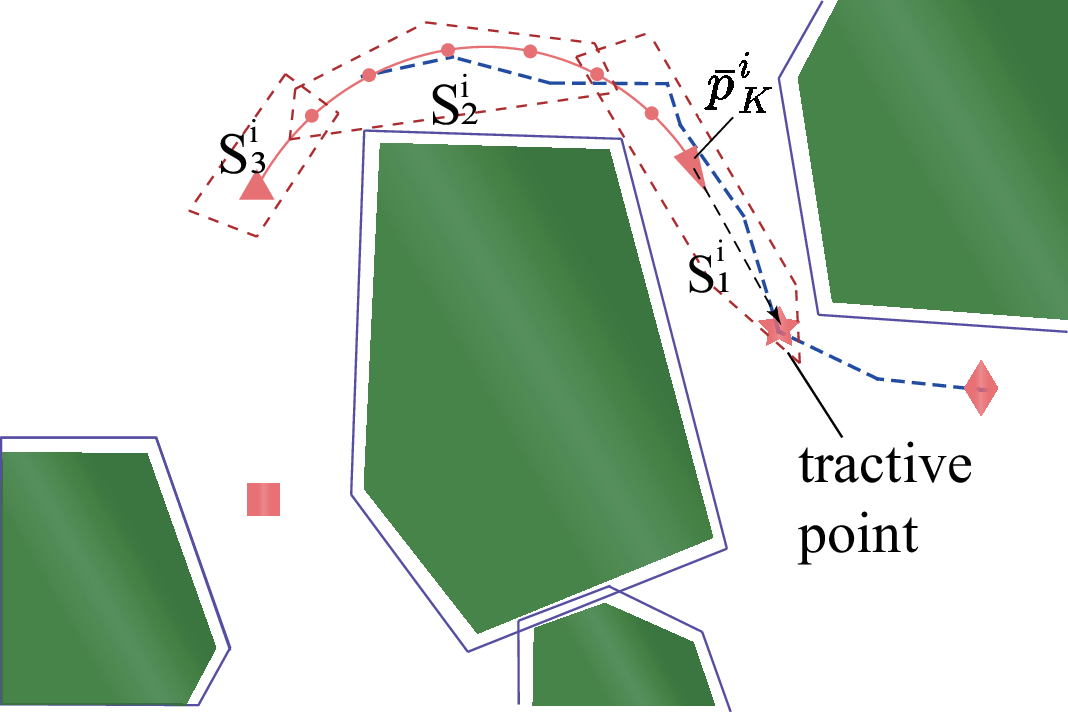}
		\label{get_segment_list}
	}
	\quad
	\subfigure[Computation of separating hyperplane.]{
		\includegraphics[width=0.45\linewidth]{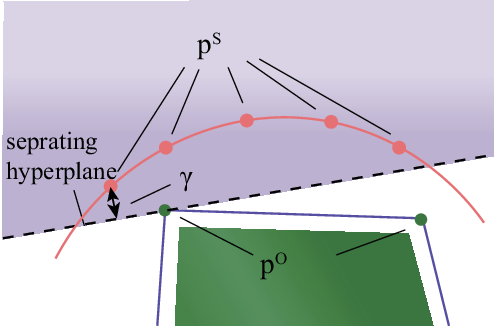}
		\label{get_separating_plane}
	}
	\subfigure[The final safe corridor.]{
		\includegraphics[width=0.45\linewidth ]{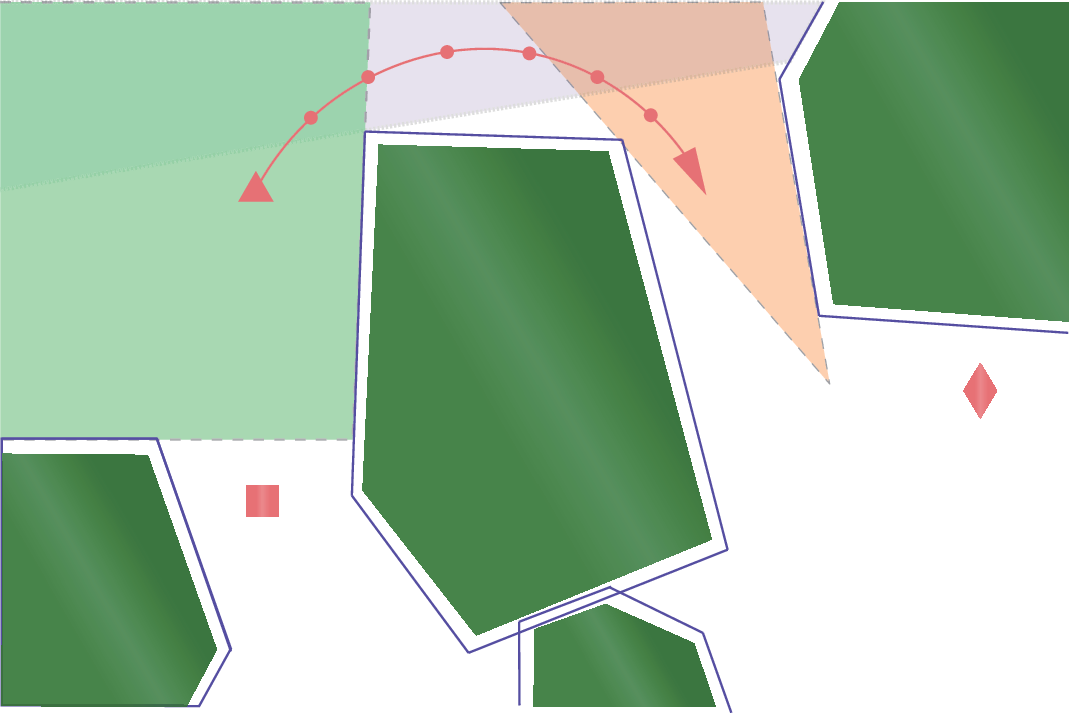}
		\label{get_corridor}
	}
	\caption{Illustration of the procedure to construct the safe corridor. Solid lines around the green area are inflated obstacles.}
	\label{corridor}
\end{figure}

\subsubsection{Computing Separating Hyperplane}
After the segmentation of EPT is done, the separating hyperplanes between the $n$-th segments~$\mathcal{S}^i_n$ and obstacles are constructed.
Since the convex hull of $\mathcal{S}^i_n$ is obstacle-free and also the obstacles are convex-shaped, a separating hyperplane exists according to separating hyperplane theorem~\cite{Boyd2004}.
Then, as shown in Fig.~\ref{get_separating_plane}, an optimization-based method can be provided as follows:
\begin{equation} \label{QP1}
	\begin{aligned}
		& \max _{ a, b , \gamma }  \; \gamma ,   \\
		{\rm s.t.},\ &  a^\mathrm{T} p^{\text{S}}  \geq \gamma + b,   \\
		&  {a}^\mathrm{T} p^{O}  \leq  b ,  \\
		&  \| a \|_{2} = 1 ,   \\
		& \gamma \geq 0,
	\end{aligned}
\end{equation}
where $a$ and $b$ determine the separating hyperplane; $p^{\text{S}} \in \mathcal{S}^i_n$ and $p^O$ denote the points of an obstacle; $\gamma$ is the margin variable.
Such an optimization can be further transformed to the following quadratic programming (QP):
\begin{equation} \label{QP2}
	\begin{aligned}
		& \min _{ a^{\prime}, b^{\prime}  }    \ \| a^{\prime} \|_2^2  \\
		{\rm s.t.},\ &  { a^{\prime} }^\mathrm{T} p^S  \geq 1 + b^{\prime},   \\
		&  { a^{\prime} }^\mathrm{T} p^O  \leq  b^{\prime},  \\
	\end{aligned}
\end{equation}
where $a^{\prime}=\frac{a}{\gamma}$ and $b^{\prime}=\frac{b}{\gamma}$.
By solving the QP in~\eqref{QP2}, the separating hyperplane can be obtained as $a=\frac{a^{\prime}}{\|a^{\prime}\|_2}$ and $b=\frac{b^{\prime}}{\|a^{\prime}\|_2}$.
Moreover, $a_{k}^{i,o}$ and $b_{k}^{i,o}$ are chosen as $a$ and $b$ to formulate the constraints in~\eqref{eq:obstacle-constraint}.

In practice, the separating hyperplanes between different segments and obstacles are constructed in a sequential manner.
Specifically, starting from the nearest obstacle to the farthest one, a separating hyperplane is built if this obstacle has a contact with the convex polytopes formed by the existing hyperplanes.
Otherwise, such an obstacle is omitted since it has been separated by the existing hyperplanes.

\begin{lemma} \label{lemma-1}
	The proposed safe corridor has the following three properties:
\begin{enumerate}
  \item If the predetermined trajectory $\overline{\mathcal{P}}^i$ is obstacle-free, then a safe corridor can be generated.
  \item The predetermined trajectory $\overline{\mathcal{P}}^i$ satisfies the constraints in~\eqref{eq:obstacle-constraint}, i.e., ${a_{k}^{i,o}}^\mathrm{T} \overline{p}_{k}^{i} \geq b_{k}^{i,o}$, $k \in \mathcal{K}$.
  \item If the constraints in~\eqref{eq:obstacle-constraint} are satisfied, then the planned trajectory $\mathcal{P}^i$ is obstacle-free.
\end{enumerate}
\end{lemma}

\begin{proof}
	First,
	since $\textbf{Conv} (\left\{ p^i_{\text{tractive}},\, \overline{p}^i_K \right\}) \cap  \tilde{\mathcal{O}} = \emptyset$, a collision-free $\tilde{\mathcal{P}}^i$ can be found.
	Consider the most conservative segmentation that the segments are formed by every adjacent two points of $\tilde{\mathcal{P}}^i$.
	Since there always exists separating hyperplanes between the points of segment and any of obstacle, the convex polytope can be formed.
	Consequently, the safe corridor can be constructed as a sequence of polytopes.
	Second, since the constraints in~\eqref{eq:obstacle-constraint} are obtained from the optimization in~\eqref{QP1} with $p^S$ chosen from the predetermined trajectory~$\overline{\mathcal{P}}^i$,
	it is clear that~$\overline{\mathcal{P}}^i$ satisfies these constraints.
	Lastly, due to the given segmentation method, $\overline{p}^i_k$ and $\overline{p}^i_{k+1}$, $k \in \tilde{\mathcal{K}}$ belong to the same segment.
	Thus, if the constraints in~\eqref{eq:obstacle-constraint} are enforced, the planned trajectory~$p^i_k$ and $p^i_{k+1}$ are restricted to a common collision-free polytope.
	Consequently, the original constraints in~\eqref{eq:oringinal-obstacle-constraint} hold and the planned trajectory is collision-free.
\end{proof}

\subsection{Deadlock Resolution} \label{sub: deadlock-resolution}
After the constraints for obstacle avoidance is transformed into~\eqref{eq:obstacle-constraint}, the overall trajectory optimization at each time step is formulated as follows:
\begin{subequations} \label{eq:final-mpc}
	\begin{align}
	  & \textbf{min}_{\{u_{k-1}^{i}, x_k^{i}, w^{i j} \}} \;  C^{i} \label{eq:objective} \\
\text { s.t. } 	  &  \eqref{eq:dynamic-constraint},\eqref{eq:input-state-constraints},  \eqref{eq:inter-constraint}, \eqref{eq:obstacle-constraint},\; k \in \mathcal{K},\, \forall j \ne i,\notag \\
	  & \;	v^i_K=\mathbf{0}_d,  \label{eq:terminal-constraint}
	\end{align}
\end{subequations}
where~$C^i$ is the objective function defined in the sequel;
constraints in~\eqref{eq:dynamic-constraint} and \eqref{eq:input-state-constraints}
stem from the complex dynamic model and constraints;
constraints in~\eqref{eq:inter-constraint} and \eqref{eq:obstacle-constraint} ensure the inter-robot and robot-obstacle collision avoidance;
constraints in~\eqref{eq:terminal-constraint} is imposed to ensure the feasibility of underlying optimization, as done similarly in \cite{Chen2022,Jesus2021}.

The objective function in~\eqref{eq:objective} is defined as $C^i = C^i_p + C^i_w$,
where~$C^i_p$ is the balanced term between the distance to the tractive point and the consecutive velocities:
\begin{equation*}
C^i_p = \frac{1}{2} Q_K \|p_{K}^{i}-p_{\text {tractive}}^{i}\|_2^2 + \frac{1}{2} \sum_{k=1}^{K-1}Q_{k}\|p_{k+1}^{i}-p_{k}^{i}\|_2^2,
\end{equation*}
where $Q_k>0$, $\forall k \in \mathcal{K}$ are weight parameters.
In addition, $C^i_w$ is the penalty term designed for deadlock resolution as follows:
\begin{equation} \label{eq:C^i_w}
	C_{w}^{i}=\sum_{j \neq i} \rho^{i j} { (\epsilon - w^{i j}) }^2,
\end{equation}
where $\rho^{i j}>0$ is the designed parameter to adjust the repulsive force from robot~$j$ to robot~$i$.
Similar design of~$\rho^{ij}$ is also adopted in~\cite{Chen2022},
which however is derived directly from the priority mechanism and the right-hand rule.

The process of updating $\rho^{i j}$ online is summarized in Algorithm~\ref{alg:deadlock-resolution}.
To begin with, the deadlock detection mechanism is introduced in Line~\ref{algline:deadlock-detecction},
where a boolean signal is introduced as~$flag^i=True$.
A deadlock is detected if the following two conditions hold:
i) $p^i_K(t) \ne p^i_\text{target}$,
$p^i_K(t)=p^i_K(t-h)$,
and $p^i_K(t)=p^i_{K-1}(t)$; and
ii) there exists another robot~$j$, such that~$w^{ij}<\epsilon$ holds.
The reason is that when a potential deadlock happens, the planned terminal position $p_K^i(t)$ of robot~$i$ remains static and the second last planned position $p^i_{K-1}(t)$ approaches~$p_K^i(t)$.
Moreover, in obstacle-dense environments, deadlocks may be induced by static obstacles instead of the blocking of other robots.
To resolve this situation, an additional condition is added to the detection process.
Namely, if there exists robot~$j$ such that~$w^{ij}<\epsilon$ holds, it indicates that robot~$i$ is blocked by at least one another robot.

\begin{algorithm}[t]
	\caption{
			DeadlockResolution()
		} \label{alg:deadlock-resolution}
	\SetKwInOut{Output}{Output}
	\SetKwInOut{Input}{Input}
	\Input{$\overline{\mathcal{P}}^j$, $p_\text{target}^j$, $\psi^j$, $\zeta^j$}
	Obtain $flag^i$ via deadlock detection \label{algline:deadlock-detecction}\;
	Obtain $\eta^i$ from~\eqref{eta^i} \label{algline:update-eta}\;
	\eIf{$\exists t > t_0$, $p_K^i(t)=p_\textup{target}^i$}{
		$\psi^i = \textbf{1}$ \label{algline:ever-reach-target}\;
	}{
		\If{ $\eta^i \geq \eta_\textup{max}$ }{
		  \If{($\zeta^i=True$) \rm{\textbf{and}} ($\psi^j \neq \textbf{3}$, $\forall j$) \rm{\textbf{and}} ($\forall j$, $\zeta^j = True$ \rm{implies}
                    $\|p_K^i-p_\textup{target}^i\|_2 < \|p_K^j-p_\text{target}^j\|_2$)}{
				$\psi^i = \textbf{3}$ \label{algline:get-highest-priority}\;
			}
			$\zeta^i = True$\;
		}
	}
	\If{($\psi^i = \textbf{3}$) \rm{\textbf{and}}  ($w^{ij} = \epsilon$, $\forall j$)}{
		$\psi^i = \textbf{2}$\; \label{algline:resturn-highest-priority}
	}
	\rm{\textbf{return}} $\rho^{ij}$ as \eqref{rho^ij}\; \label{algline:get-rho}
\end{algorithm}

Afterwards, it is followed by the update of a key parameter~$\eta^i(t)$ as in Line~\ref{algline:update-eta},
which represents the magnitude of the deadlock.
And it is updated as follows:
\begin{equation} \label{eta^i}
	\eta^i(t) = \left\{
	\begin{array} {rll}
		&\eta^i(t-h) + \Delta \eta , &\mbox{if}\; flag^i = True; \\
		&\eta_\text{max}           , &  \mbox{if}\; (flag^i = True \\
		&                           &\text{and} \ \eta^i(t-h) \geq \eta_\text{max}); \\
		&0                         , &\mbox{if}\; \  \forall j \ne i, w^{i j}(t-h) = \epsilon;\\
		&\eta^i(t-h)               , &\text{else},
	\end{array}	\right.
\end{equation}
where $\Delta \eta >0$ is a design parameter and $\eta_\text{max}$ is the chosen upper bound of $\eta^i(t)$.

Based on these parameters, the dynamic priority of robot~$i$ denoted by~$\psi^i$ is determined.
A higher priority leads to more egoistic behaviors when avoiding the others.
Initially, all robots have the medium priority as $\psi^i = \textbf{2}$.
Furthermore, if a robot has reached its target, it obtains the lowest priority $\psi^i = \textbf{1}$ as in Line~\ref{algline:ever-reach-target}.
Moreover, if $\eta^i \geq \eta_\text{max}$, i.e., the deadlock measurement has reached the highest magnitude, robot~$i$ is chosen as the candidate that obtains the Priority~\textbf{3} and consequently $\zeta^i = True$.
More importantly, any robot~$i$ actually obtains Priority~\textbf{3}  as in Line~\ref{algline:get-highest-priority} if the following three conditions are all satisfied:
i) it is the one of the candidates, i.e., $\zeta^i = True$;
ii) the highest priority is still vacant;
and iii) compared with other candidates, robot~$i$ has the shortest distance to the target.
By this mechanism, the highest priority is assigned in a distributed way and only one robot is entitled the highest priority.
Notably, the highest priority of one particular robot is not permanent as once it has escaped from the deadlock and has no contact with any others, its highest priority is deprived as in Line~\ref{algline:resturn-highest-priority}.
In other words, the priorities are dynamically changing based on the status of the robots.

After the priorities are determined, the parameter~$\rho^{ij}$ in~\eqref{eq:C^i_w} is computed as follows:
\begin{equation} \label{rho^ij}
	\rho^{ij} = \left\{
	\begin{array} {rll}
		&\frac{\rho_0}{\gamma^{i j}} \, e^{(\eta^i \, \sin \theta^{i j})} ,  &\mbox{if}\; \psi^i = \psi^j; \\
		&\rho_\text{min}         ,                  &\mbox{if}\;   \psi^i > \psi^j;\\
		&\rho_\text{max}         ,                  &\mbox{if}\;   \psi^i < \psi^j,
	\end{array}	\right.
\end{equation}
where the $\theta^{ij}$ is the angle in $x-y$ plane between the projection of the line $\overline{p}^i_K$  to the point~$p^i_{\text{tractive}}$ and to the point~$\overline{p}^j_K$;
$\gamma^{i j}$ is used to adjust the penalty in a proper range,
which is updated as~$\gamma^{i j}(t)=(1-\beta) \gamma^{i j}(t-h) + \beta\, \epsilon\, w^{i j} (t-h)$,
where $\beta \in  [0.5,\, 1)$ and $\gamma^{i j}(t_{0}) = {\epsilon}^2$;
$\rho_0$, $\rho_\text{min}$ and $\rho_\text{max}$ are coefficients which satisfy
$\rho_\text{max} \gg \frac{\rho_0}{\gamma^{ij}} \gg \rho_\text{min}>0$.

In general, since most agents have Priority $\textbf{2}$, they adopt a right-hand rule as $\rho^{ij}=\frac{\rho_0}{\gamma^{i j}} \, e^{(\eta^i \, \sin \theta^{i j})}$ to coordinate with each other.
More specifically, when the deadlock is detected as described above and $\eta^{ij}>0$, the right-hand rule is activated.
For the robot $j$ on the left, i.e., $\theta^{ij}>0$, it holds that~$e^{(\eta^i \, \sin \theta^{i j})} > 1$ which indicates that the repulsive effect from robot~$j$ and robot~$i$ is increased, which keeps robot~$i$ away from robot~$j$.
On the contrary, if robot~$i$ approaches robot~$j$ on its right hand, the coordination strategy is different from the previous case.
More specifically, when the deadlock magnitude~$\eta^i$ exceeds the designated upper bound~$\eta_\text{max}$, its priority is changed as described earlier.
Thus, the robot with a higher priority has a much lower penalty coefficient as $\rho^{ij} = \rho_\text{min}$, which means that it can approach and expel other robots that have lower priorities.
This mechanism is critical in obstacle-dense environments since the right-hand rule may be invalidated when obstacles block the way along with other robots.
An example of such situation is shown in Fig.~\ref{deadlock-in-obstacle}.
However, the robot with the highest priority can neglect the repulsive effect from other robots and continue moving in order to break this logjam.

\begin{figure}[t!]
	\centering
	\includegraphics[width=0.9 \linewidth]{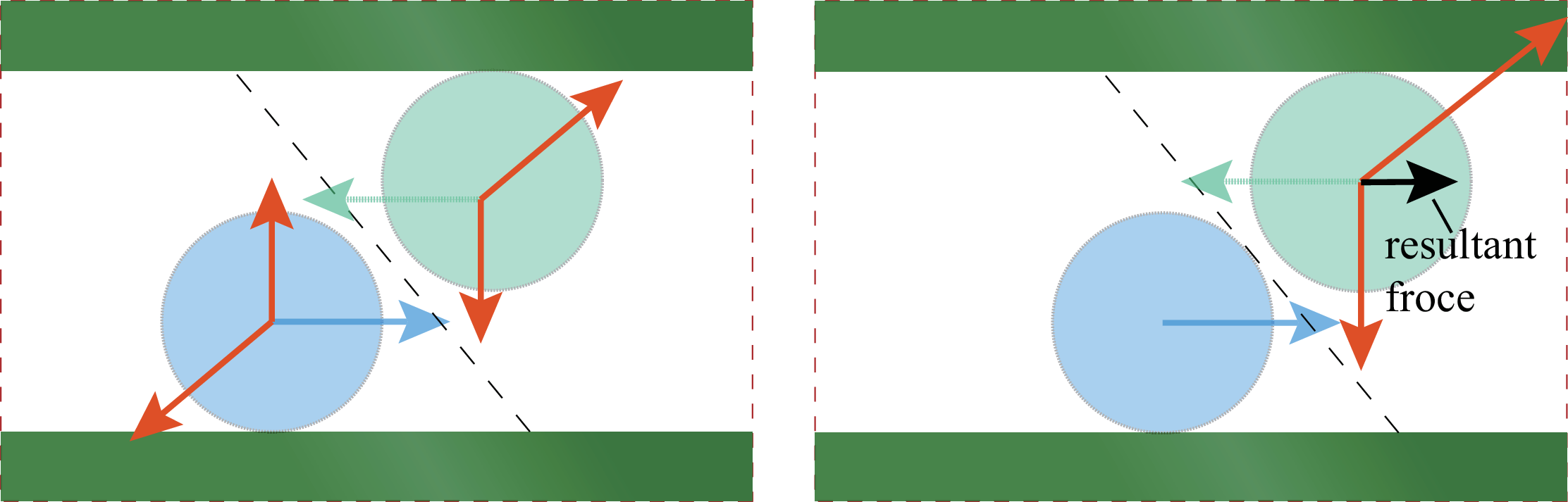}
	\caption{
		\textbf{Left}: the existence of static obstacles can invalidate the right-hand rule. \textbf{Right}: after obtaining the highest priority, the blue robot ignores the repulsive effect from the green robot and continues its motion.
	}
	\label{deadlock-in-obstacle}
\end{figure}

\subsection{Overall Trajectory Optimization Algorithm}
Given the optimization problem in~\eqref{eq:final-mpc} and the deadlock resolution scheme described in the previous section,
the overall trajectory optimization is summarized in Algorithm~\ref{AL:IMPC-OB}.

To begin with, the predetermined trajectory is initialized in Line~\ref{algline:impc-init}.
After initialization, in the main loop, each robot runs the same planning algorithm in a parallel and distributed way as in Line~\ref{algline:impc-each-robot}.
Therein, the packed data $Data^i = \{ \overline{\mathcal{P}}^i,\, p_\text{target}^i,\, \psi^i,\, \zeta^i\}$ is exchanged among the robots, which are essential for formulating and solving the final optimization and enforcing the deadlock resolution scheme, as in Line~\ref{algline:impc-commu}.
In particular, it is followed by the safe corridor construction as descried in Sec.~\ref{subsection:safe-corrdior}  and Line~\ref{algline:impc-obstacle-cons}.
Thereafter, Algorithm~\ref{alg:deadlock-resolution} is followed to resolve potential deadlocks and further obtain the penalty parameter~$\rho^{ij}$ in Line~\ref{algline:deadlock-resolution}.
Thus, the final optimization in~\eqref{eq:final-mpc} is formulated and solved in Line~\ref{algline:convex-programming}.
The resulting trajectory is executed via sending it to the lower-level feedback controller (Line~\ref{algline:impc-execute}).
This procedure is repeated until all robots have reached their target positions.

\begin{algorithm}[t] \label{AL:IMPC-OB}
	\caption{
		The Complete Algorithm
		}\label{algorithm}
	\SetKwInOut{Input}{Input}
	\Input{$p^i(t_0)$, $p^i_\text{target}$,$\mathcal{O}$}
	$\overline{\mathcal{P}}^i(t_0)=\left\{ p^{i}(t_0), \ldots, p^{i}(t_0) \right\}$\label{algline:impc-init}\;
	\While{ \rm{not all robots at target} \label{algline:impc-while}}
	{
		\For{$i\in \mathcal{N}$ \label{algline:impc-each-robot} \rm{concurrently} }{

			obtain $Data^j$, $j \ne i$ via communication\label{algline:impc-commu}\;

			obtain $a_{k}^{i,o}, b_{k}^{i,o}$ via constructing safe corrdior\label{algline:impc-obstacle-cons}\;

			$\rho^{ij} \leftarrow \text{DeadlockResolution}(Data^j)$\label{algline:deadlock-resolution}\;

			obtain $\mathcal{P}^i(t)$ from optimization \eqref{eq:final-mpc}\label{algline:convex-programming}\;

			send $\mathcal{P}^i(t)$ to lower-level controller\label{algline:impc-execute}\;
		}
		$t \leftarrow t+h$\;
	}
\end{algorithm}

The overall computational complexity of the proposed method is analyzed as follows.
The computational cost consists of four main parts: i) finding the tractive points, ii) segmentation of the EPT, iii) computing the separating hyperplanes and iv) solving the final optimization.
To obtain a tractive point, the collision detection between line-polytopes is carried out less than $n_\text{path} \cdot n_\text{proximal}$ times, where $n_\text{path}$ is the number of points in $\Gamma^i$ and $n_\text{proximal}$ is the number of proximal obstacles around the robot.
In general, the collision detection between line-polytopes is a special case of the collision detection between polytopes, which can be done with complexity $\mathcal{O}(n_1+n_2)$ where $n_1$ and $n_2$ are the numbers of vertices in these two polytopes~\cite{Gilbert1988}.
The segmentation of the EPT is finished after checking the collision between two polytopes for $2 K  n_\text{proximal}$ times.
To compute the separating hyperplanes, the worst situation induces $K$ segments and consequently hyperplanes with $n_\text{proximal}$ obstacles are obtained after solving $K n_\text{proximal}$ QPs.
Each of these QPs has $d+1$ variables and $d+n+1$ constraints where $n$ is the number of vertices belonging to the underlying obstacles.
In the final optimization, a quadratically constrained quadratic programming (QCQP) with $Kd+N-1$ variables and $(N+n_\text{priximal}+1)K+2N-1$ constraints is formulated.
Owing to mature optimization tools, e.g., \cite{cvxopt}, the underlying optimization can be resolved within tens of milliseconds.

\subsection{Feasibility Guarantee}

Different from most existing optimization-based methods, e.g., \cite{Jesus2021,Toumieh2022}, the proposed planning algorithm guarantees the recursive feasibility of the optimization problem, as proven in the following theorem.
\begin{theorem}  \label{recursive feeasible}
	If all robots are collision-free with each other and any obstacle initially,
	they remain so under Algorithm~\ref{algorithm}.
\end{theorem}

\begin{proof}
	At the beginning time~$t_0$, it is clear that the predetermined trajectory $\overline{\mathcal{P}}^i(t_0)$ is a feasible solution of optimization \eqref{eq:final-mpc} for all robots.

	As stated in Lemma~\ref{lemma-1}, if trajectory planning at $t-h$ is feasible,
	the planned trajectory at $t-h$ is obstacle-free, based on which a safe corridor can be derived.
	Thus, the final optimization \eqref{eq:final-mpc} can be formulated.
	Given the feasible solution at the previous time step,
	i.e., $u^i_{k-1}(t-h)$ and $x^i_k(t-h)$ for $k \in \mathcal{K}$,
	a feasible solution $x^i_k(t)=x^i_{k+1}(t-h)$, $u^i_k(t)=u^i_{k+1}(t-h)$ can be
	derived where we enforce $x^i_{K}(t)=x^i_{K}(t-h)$ and $u^i_{K-1}(t)=\mathbf{0}_d$.
	More specifically,
	as the result of optimization at time step $t-h$,
	$x^i_{k+1}(t-h)$ and $u^i_{k}(t-h)$ with $k \in \tilde{\mathcal{K}}$, satisfy the constraints in~\eqref{eq:dynamic-constraint} and \eqref{eq:input-state-constraints} naturally.
	In addition, since $x^i_{K}(t)=x^i_{K}(t-h)$, $v^i_{K-1}(t)=v^i_K(t-h)=\mathbf{0}_d$ and $u^i_{K-1}(t)=\mathbf{0}_d$ hold, $x^i_{K}(t)$ and $u^i_{K-1}(t)$ also satisfy these constraints.
	Meanwhile, as $x^i_{K}(t)=x^i_{K}(t-h)=x^i_{K-1}(t)$ holds, it is evident that the constraint \eqref{eq:terminal-constraint} holds as well.
	Then, as proven similarly in Theorem~3 of~\cite{Chen2022},
	the constraints in~\eqref{eq:inter-constraint} are satisfied as well.
	Lastly,
	as stated in the first and third properties of Lemma~\ref{lemma-1},
	the feasible solution at time $t-h$ is obstacle-free and the associated safe corridor can be determined.
	Thus, the constraint in~\eqref{eq:obstacle-constraint} is feasible
	and the optimization in~\eqref{eq:final-mpc} is feasible in a recursive way.

	Since the initial optimizations as well as these successive ones are feasible,
	the enforced constraints are satisfied.
	Due to the third property in Lemma~\ref{lemma-1} and Theorem~1 in \cite{Chen2022},
	the collision avoidance among robots and between robots and obstacles are ensured.
    This completes the proof.
\end{proof}

\begin{figure} [t]
	\centering
	\includegraphics[width=0.95\linewidth]{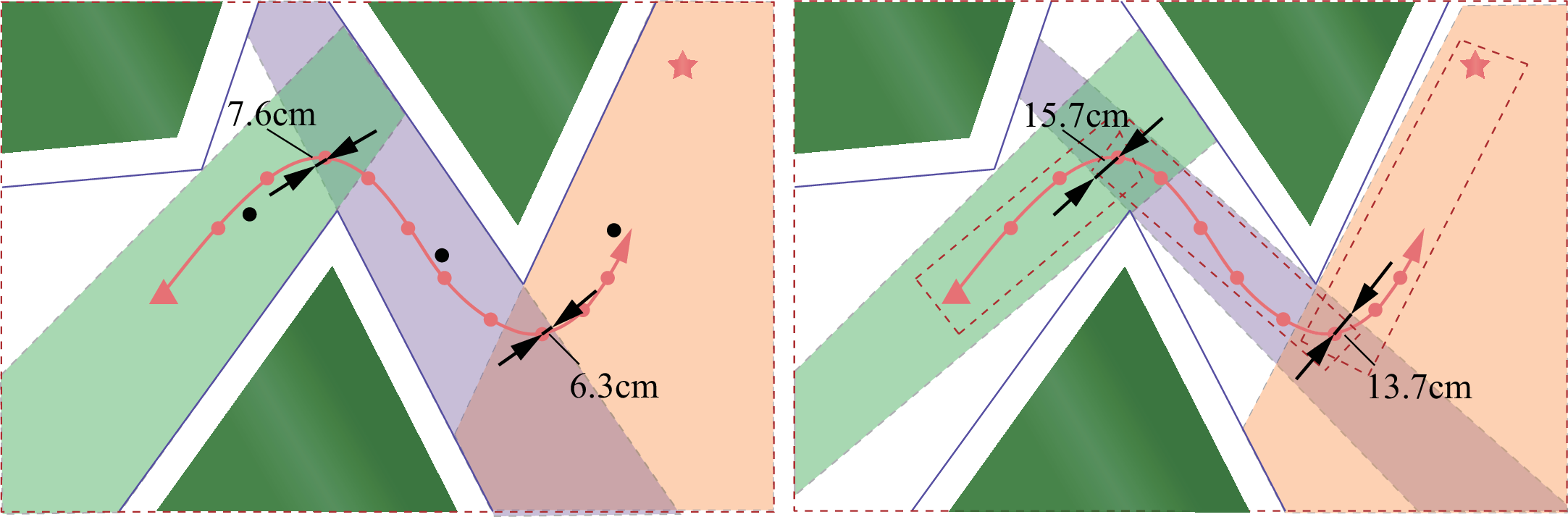}
	\caption{
		\textbf{Left}: the safe corridor generated by IRIS \cite{Deits2015}. \textbf{Right}:  ours safe corridor. In comparison, points of predetermined trajectory generated by our method are further away from the obstacles boundary, yielding a larger safety margin.
		}
	\label{corridor-contrast}
\end{figure}

\begin{figure} [t]
	\centering
	\subfigure{
		\includegraphics[width=0.47\linewidth]{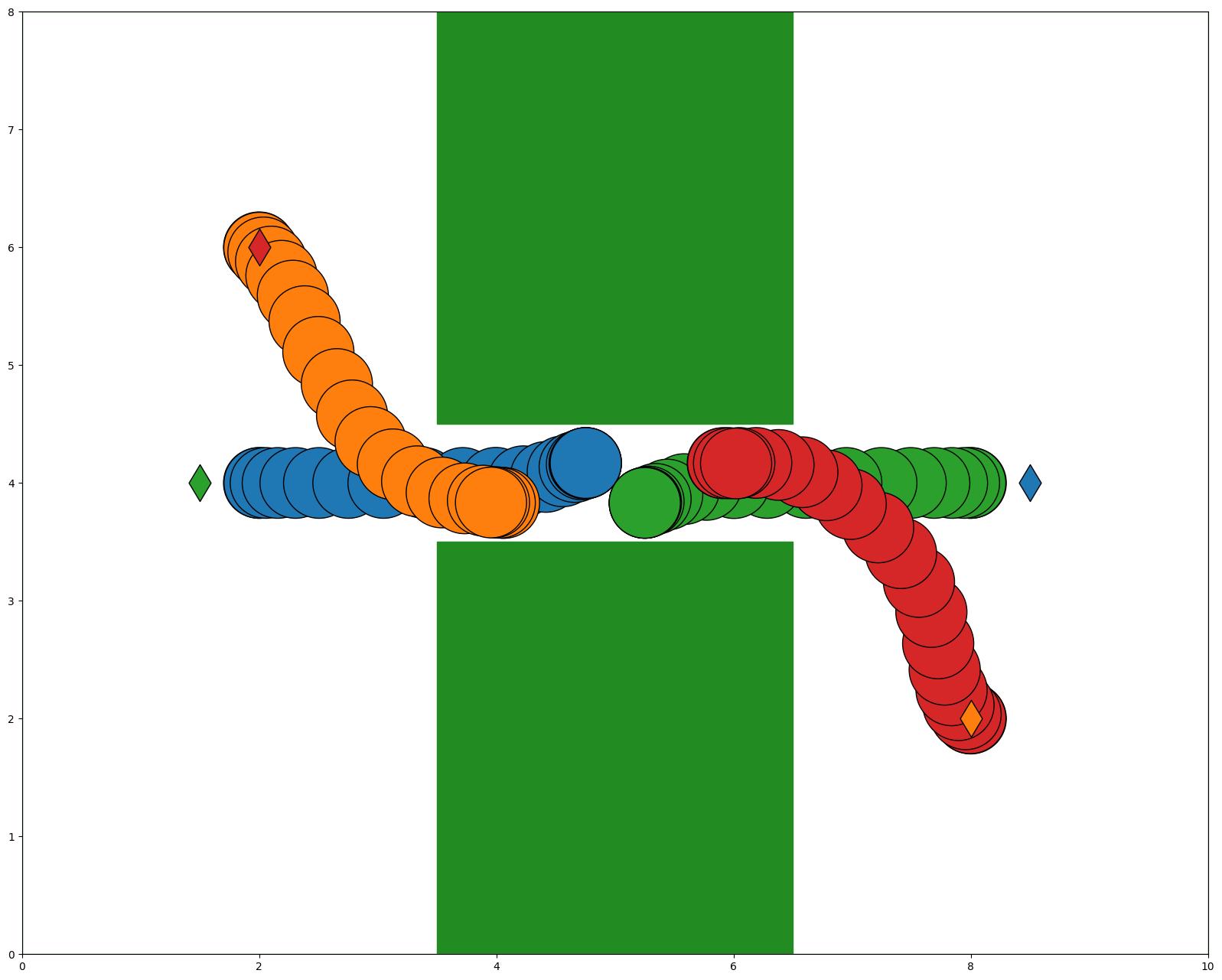}
		\includegraphics[width=0.47\linewidth]{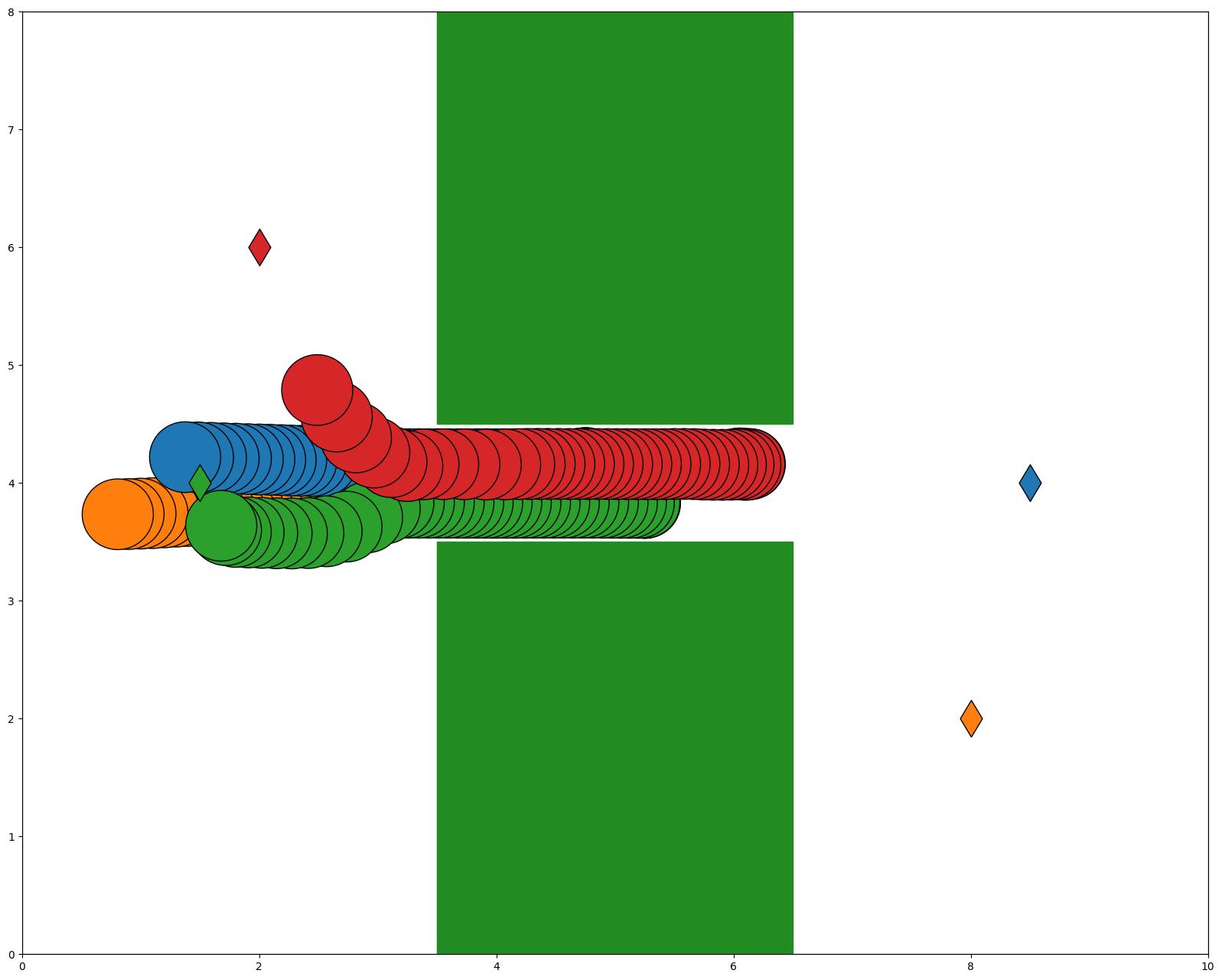}
	}
	\caption{
		\textbf{Left}: four robots remain being blocked by each other after only the right-hand rule is activated. \textbf{Right}: the green robot navigates through the corridor safely by expelling the blue and orange robots, under  the proposed dynamic-priority scheme.
		}
	\label{deadlock-resolution-in-obstacle}
\end{figure}

\section{Numerical Simulations and Experiments}

This section validates the performance of the proposed algorithm via numerical simulations and hardware experiments.
The algorithm is implemented in Python3 and publicly available at https://github.com/PKU-MACDLab/IMPC-OB.
All tests are performed on a computer with Intel Core i9 @3.2GHz and the team of robots is simulated using multiprocessing.
CVXOPT \cite{cvxopt} is used for solving the trajectory optimization, and OMPL \cite{ompl} for ABIT$^\star$ path planning.
The simulation and experiment videos can be found at https://youtu.be/Er7v9zQx784.

\subsection{Numerical Simulations and Comparison}

In the following simulations,
the main parameters of robots are chosen as follows:
$r_a=0.3{\rm m}$;
$a_{\text{max}}=2{\rm m/s^2}$;
$v_{\text{max}}=3{\rm m/s}$;
$h=0.15$s.

To begin with, the safe corridor generation is compared with IRIS~\cite{Deits2015}.
For IRIS, the centers of segmented points (shown as black points in Fig.~\ref{corridor-contrast}) are initialized to generate the corridor.
On average, IRIS takes $21$ms to compute a safe corridor in contrast to $4$ms via our method.
As shown in Fig.~\ref{corridor-contrast}, IRIS generates a larger interaction area than ours.
However, the trajectories generated by our method have a larger distance to the boundary of the safe corridors, yielding a larger safety area around critical regions.
This can be beneficial for ensuring the feasibility of underlying trajectory optimization.

To further validate our method, a rather extreme scenario is considered where four robots encounter in a narrow passage which only accommodates one robot.
As shown in Fig.~\ref{deadlock-resolution-in-obstacle}, even with the activated right-hand rule in \cite{Chen2022}, four robots are stuck in this passage as they all have the same priority.
However, since under the proposed algorithm the green robot obtains the highest priority after its parameter~$\eta^i$ reaches $\eta_\text{max}$, it can expel the blue and orange robots to continue its motion through the passage and reach the target.
Afterwards, the red robot follows the blue robot through the passage.

\begin{table} [t]
  \begin{center}
	\caption{Comparison With State-of-art Methods. (Safety: no collision occurs. $T_t[{\rm s}]$: transition time. $L_t[{\rm m}]$: the length of transition. $T_c[{\rm ms}]$: mean computation time per replanning.) }
	\label{table: comparison}
	\begin{tabular}{lllllll}
		\toprule
		&method &Safety &$T_t$ &$L_t$ &$T_c$ \\
		\midrule
		\multirow{5}{*}{``Forest"}&Ego-swarm \cite{Zhou2021} &No  &9.2  &105.8 &9.6 \\
		&MADER \cite{Jesus2021}    &No  &22.3 &111.1 &104.0 \\
		&LSC \cite{Park2022}       &Yes &22.3 &114.3 &53.2 \\
		&Ours                      &Yes &8.5  &102.2 &85.6 \\
		\midrule
		\multirow{5}{*}{``H"} &Ego-swarm \cite{Zhou2021} &No  &7.5  &66.6 &10.2 \\
		&MADER \cite{Jesus2021}    &No  &14.2 &71.5 &116.7 \\
		&LSC \cite{Park2022}       &Yes &- &- &62.3 \\
		&Ours                      &Yes &7.5  &65.8 &75.2 \\
		\bottomrule
	\end{tabular}
\end{center}
\end{table}

\begin{figure} [t]
	\centering
	\subfigure[Ego-swarm\cite{Zhou2021}]{
		\includegraphics[width=0.4\linewidth]{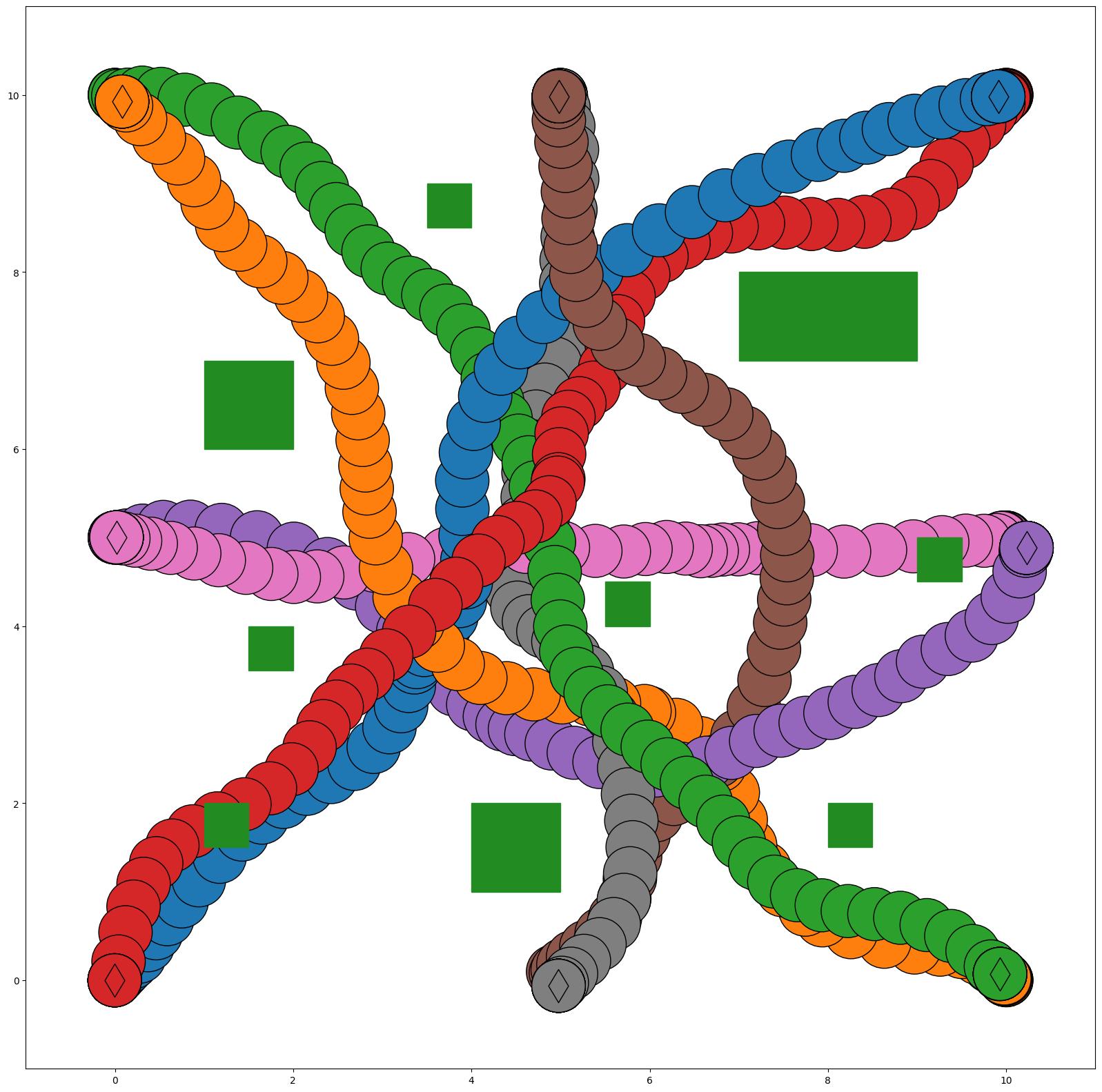}
	}
	\subfigure[MADER\cite{Jesus2021}]{
		\includegraphics[width=0.4\linewidth]{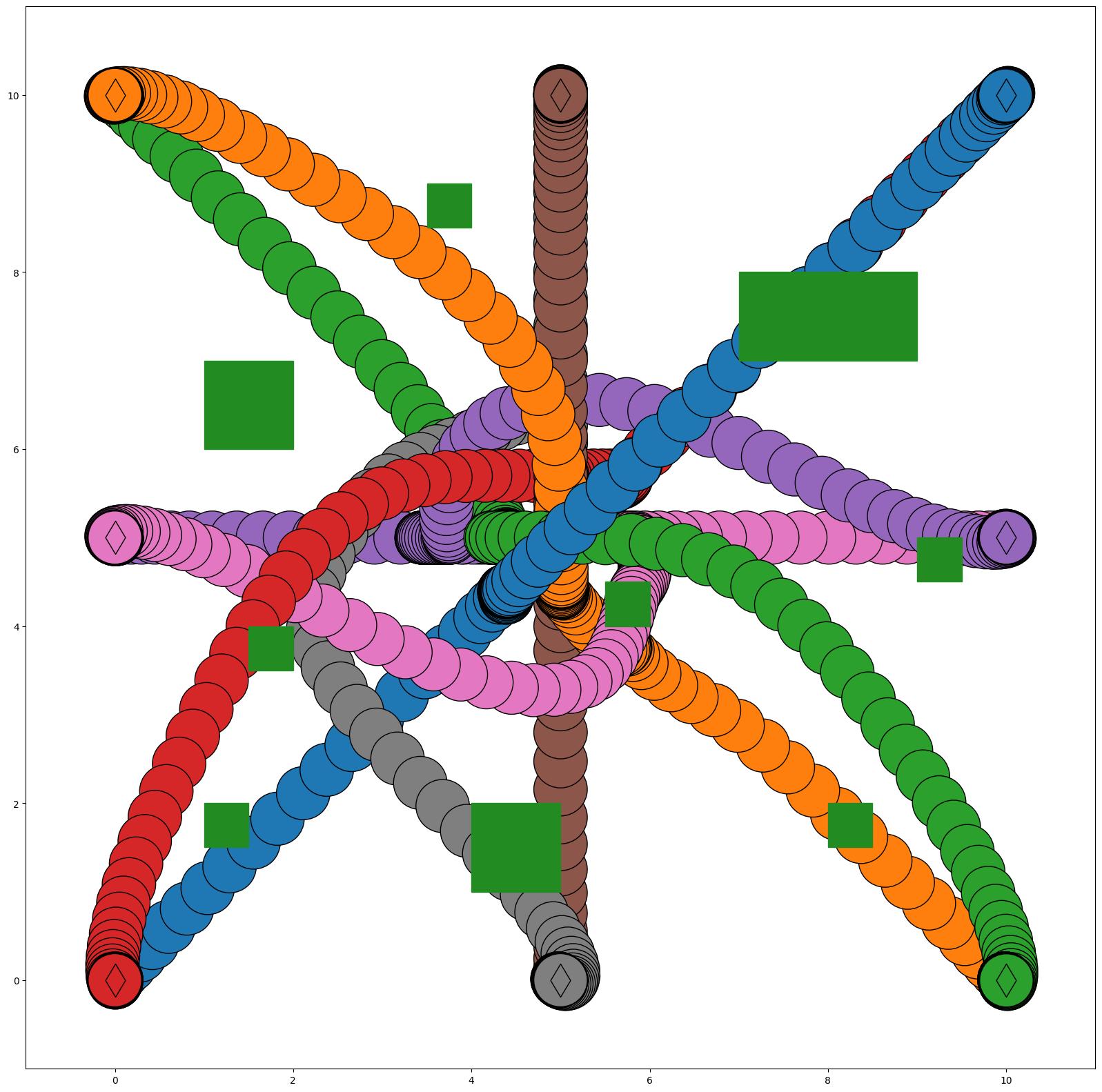}
	}
	\quad
	\subfigure[LSC\cite{Park2022} ]{
		\includegraphics[width=0.4\linewidth]{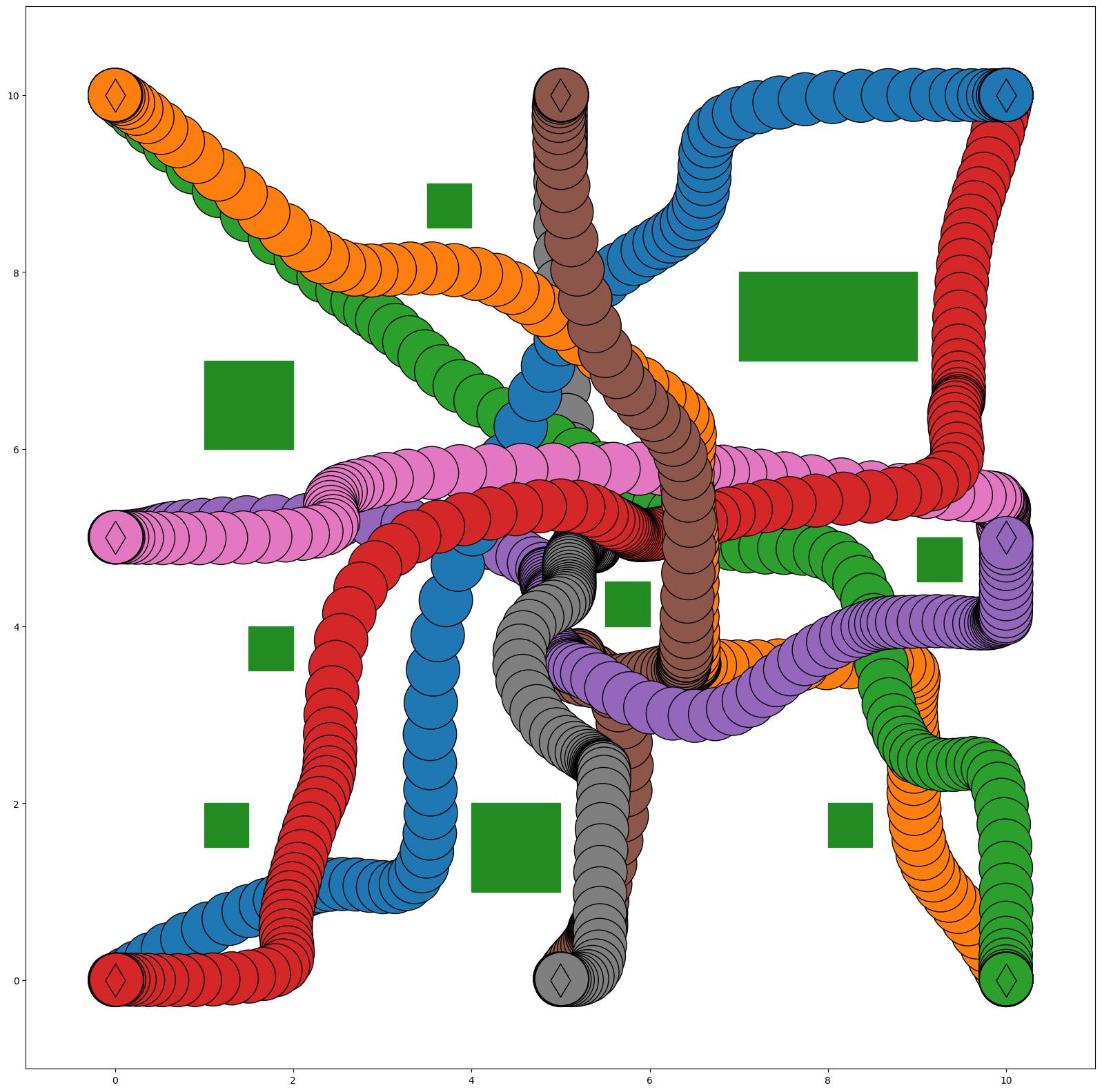}
	}
	\subfigure[Ours]{
		\includegraphics[width=0.4\linewidth]{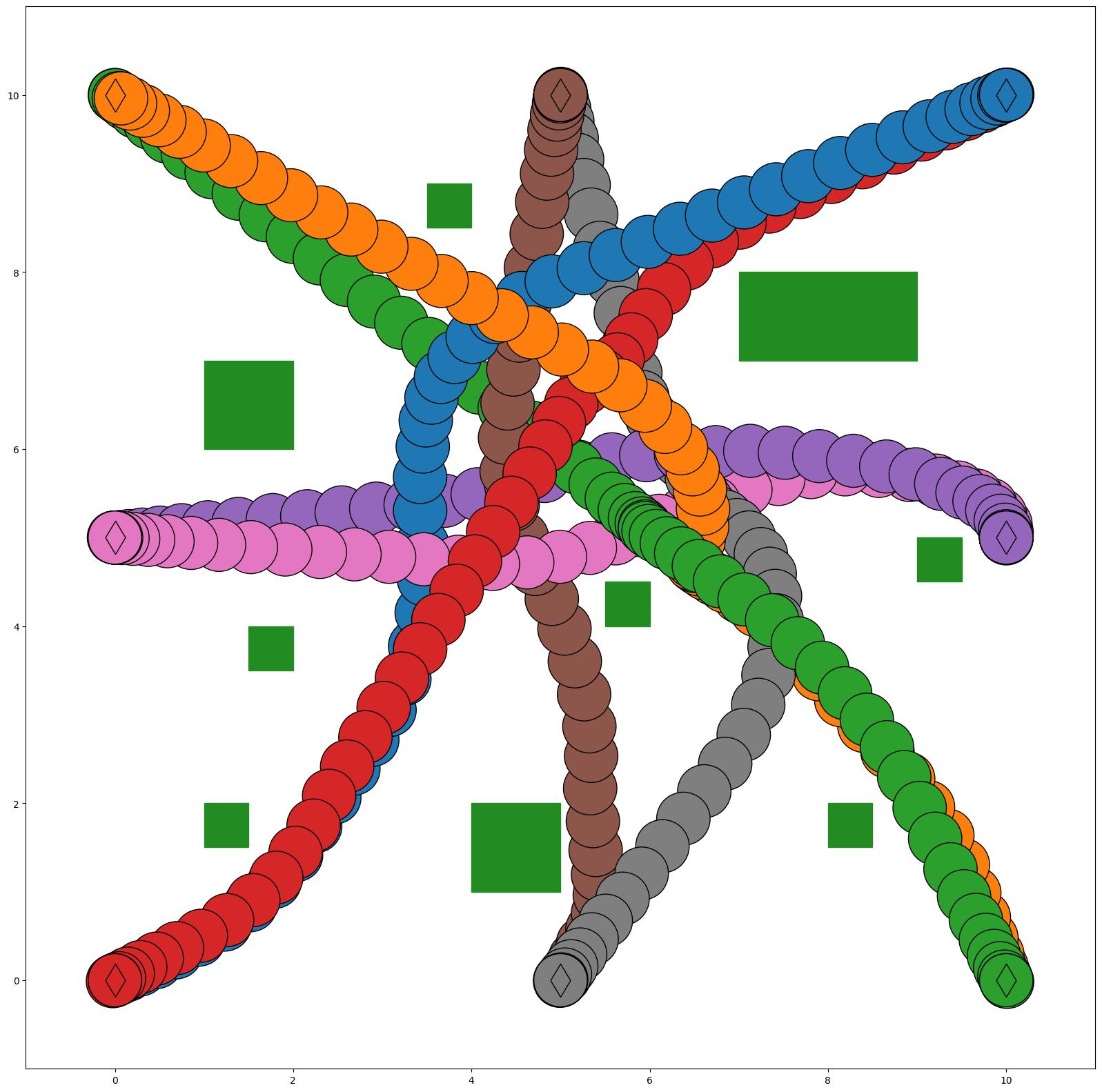}
	}
	\caption{Comparison of four planners in the ``Forest" scenario.}
	\label{Forest-transition}
\end{figure}

\begin{figure}
	\centering
	\subfigure[Ego-swarm\cite{Zhou2021}]{
		\includegraphics[width=0.4\linewidth]{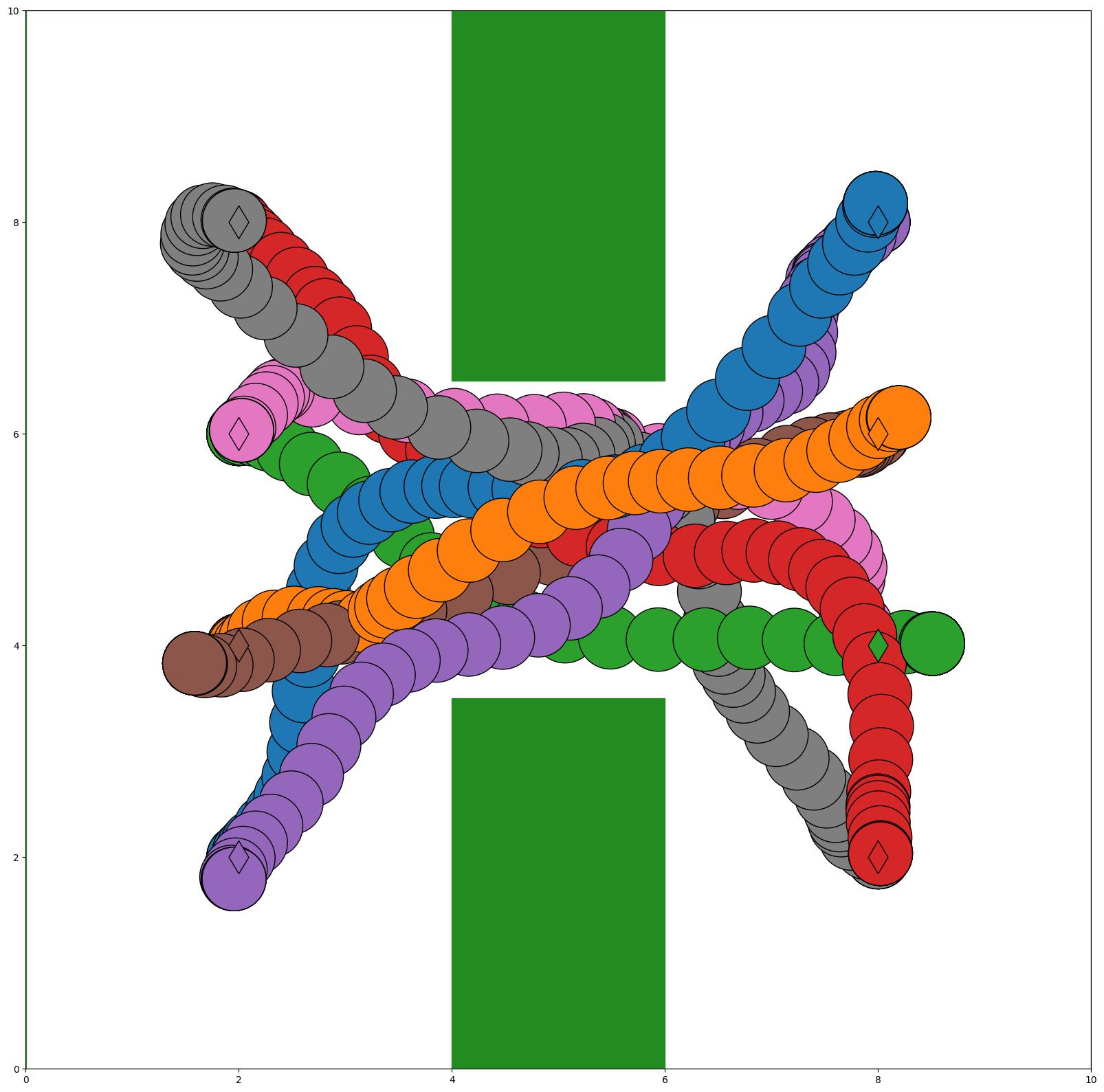}
	}
	\subfigure[MADER\cite{Jesus2021}]{
		\includegraphics[width=0.4\linewidth]{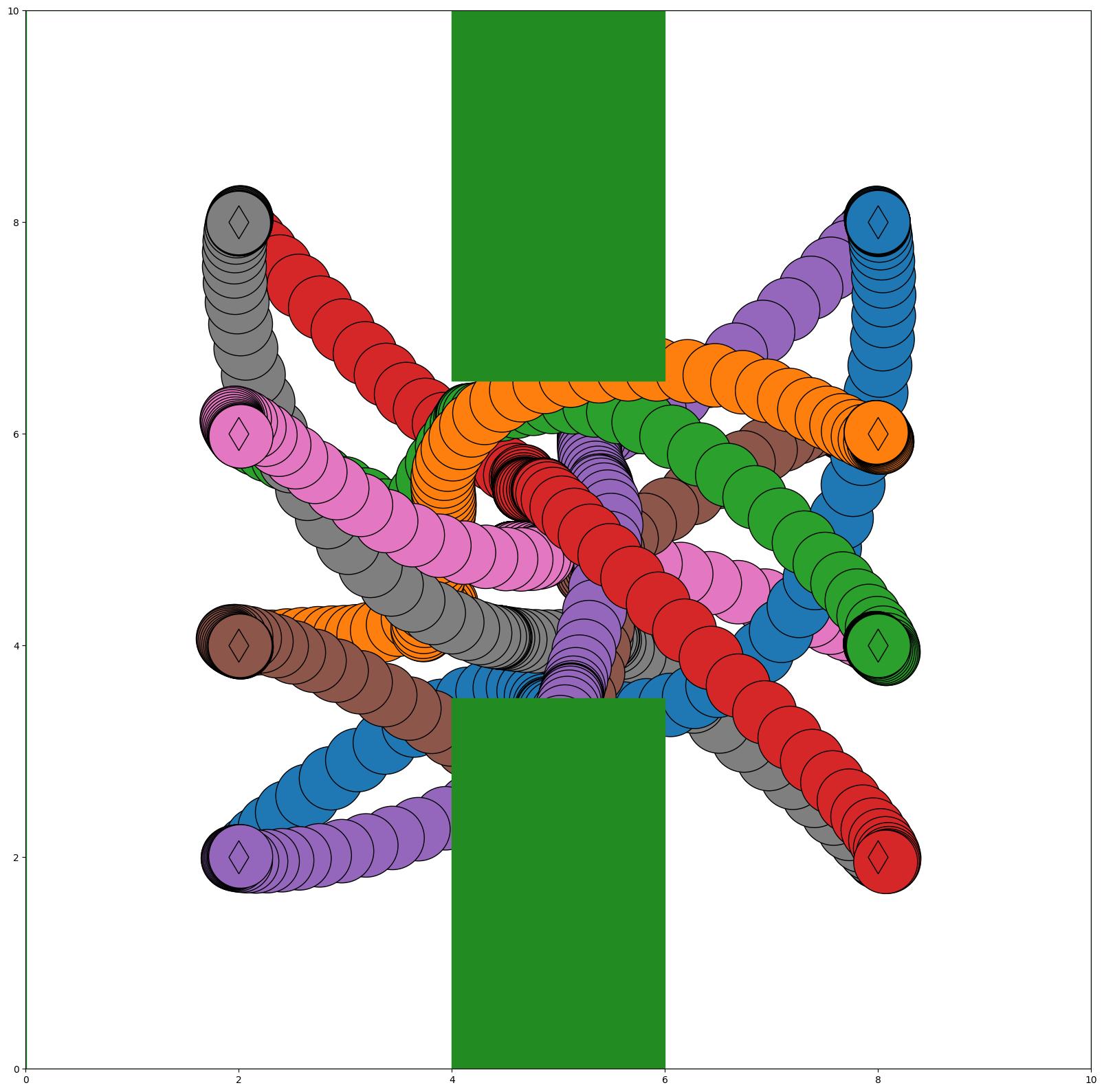}
	}
	\quad
	\subfigure[LSC\cite{Park2022} ]{
		\includegraphics[width=0.4\linewidth]{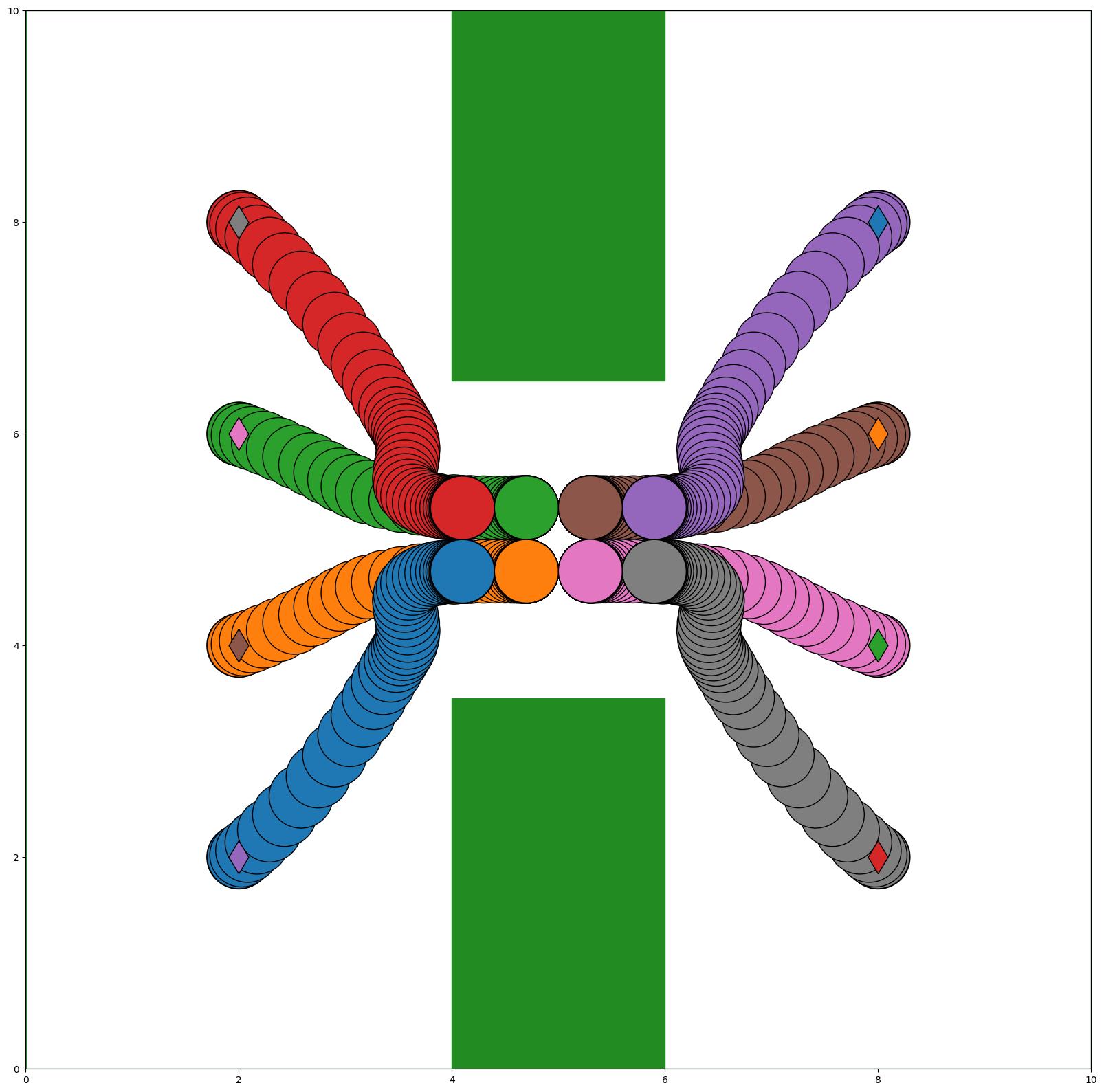}
	}
	\subfigure[Ours]{
		\includegraphics[width=0.4\linewidth]{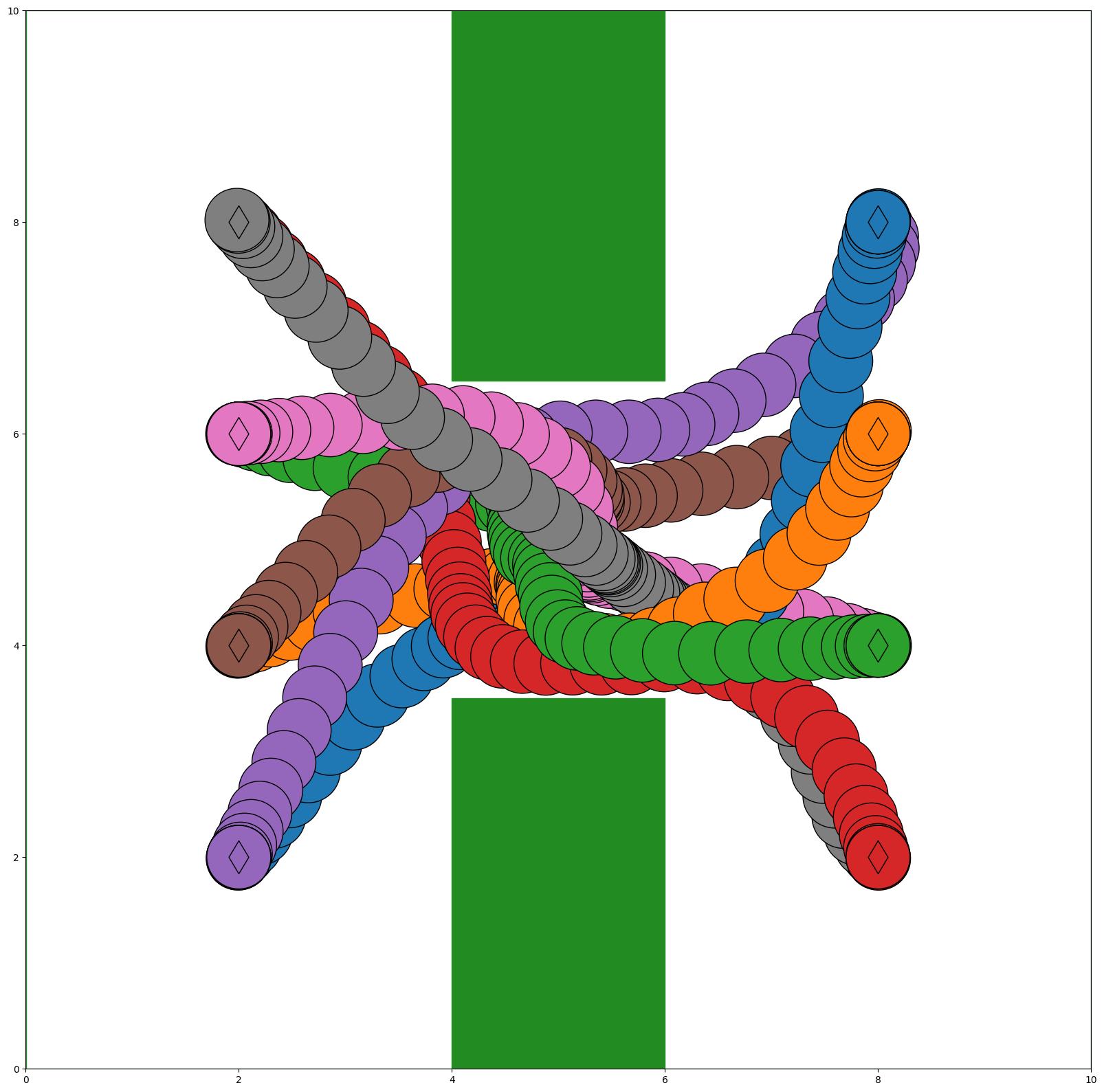}
	}
	\caption{Comparison of four planners in the ``H'' scenario.}
	\label{H-transition}
\end{figure}

\begin{figure}[t]
	\centering
	\includegraphics[width=0.98\linewidth]{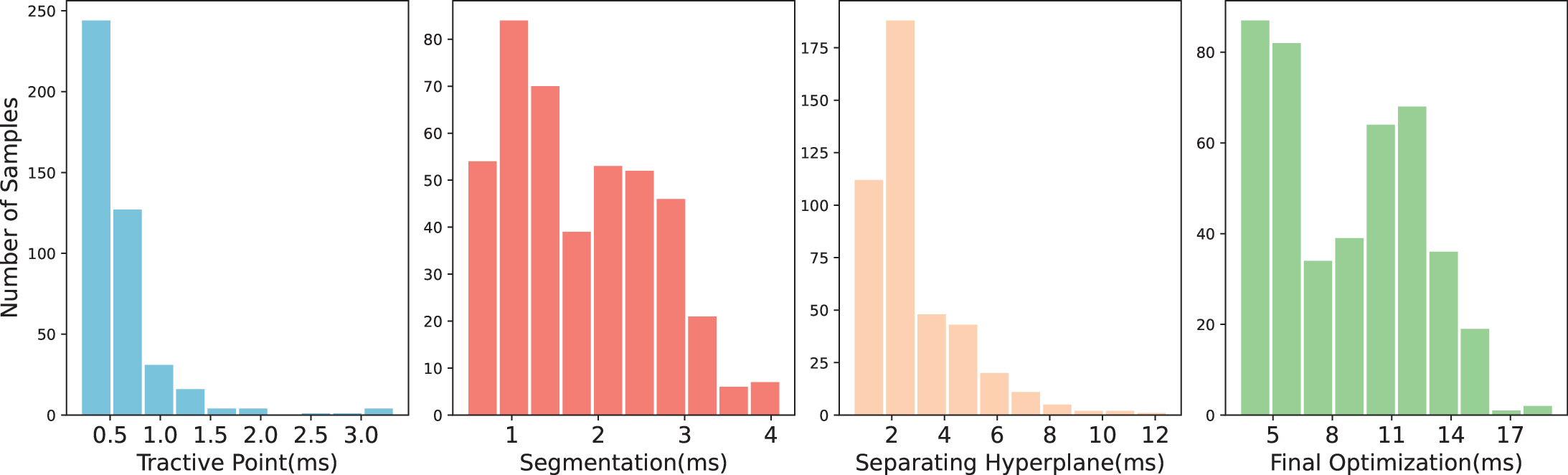}
	\caption{
		The distribution of computation cost in four main steps of the replanning.
		}
	\label{time-cost}
\end{figure}

Furthermore, comparisons with Ego-swarm~\cite{Zhou2021}, MADER~\cite{Jesus2021} and LSC~\cite{Park2022} are made within the ``Forest" and ``H" environments.
Sampled trajectories of these scenarios are illustrated in Figs.~\ref{Forest-transition} and~\ref{H-transition}, respectively.
The results are summarized in Table~\ref{table: comparison}.
Ego-swarm has an impressive computation time in addition to a relatively smooth and fast trajectories.
Unfortunately, it cannot guarantee the collision avoidance, as it adopts an unconstrained optimization in planning.
For MADER, the optimization under the strict constraints guarantees the avoidance between robots.
However, regarding obstacle avoidance, MADER is not well performed in an obstacle-dense environment as several collisions appear.
LSC has the feasibility guarantee which ensures the safety of robots, but its trajectory is not soft since the cubic safe corridor.
Moreover, deadlocks occurs in ``H"-transition as the heuristic deadlock resolution method in LSC is ineffective in this scenario.
In contrast, the adaptive right-hand rule is leveraged by the proposed method, resulting in the right-hand rotations in the bottleneck of the ``H".
Moreover, it can be seen that the proposed planner generates a much faster and smoother trajectories as the result of high-efficiency safe corridor.

Last but not the least, the time cost distribution of the four main steps: finding the tractive point, segmentation of the EPT, computing the separating planes and solving the final optimization~\eqref{eq:final-mpc} is shown in Fig.~\ref{time-cost}.
The result is collected from the ``Forest'' scenario shown in Fig.~\ref{Forest-transition}.
It can be seen that finding the tractive point takes a negligible time and computing separating planes consumes more time than the segmentation of EPT.
In contrast, most of the planning time is spent on solving the final optimization for the desired trajectory.

\begin{figure}[t!]
	\centering
	\subfigure[\textbf{Left}: final trajectories of the 8-robot swarm where the ellipsoids indicate the safety area around robots. \textbf{Right}: the inter-robot distances over all time. ]{
		\includegraphics[width=0.41\linewidth,height=0.40\linewidth]{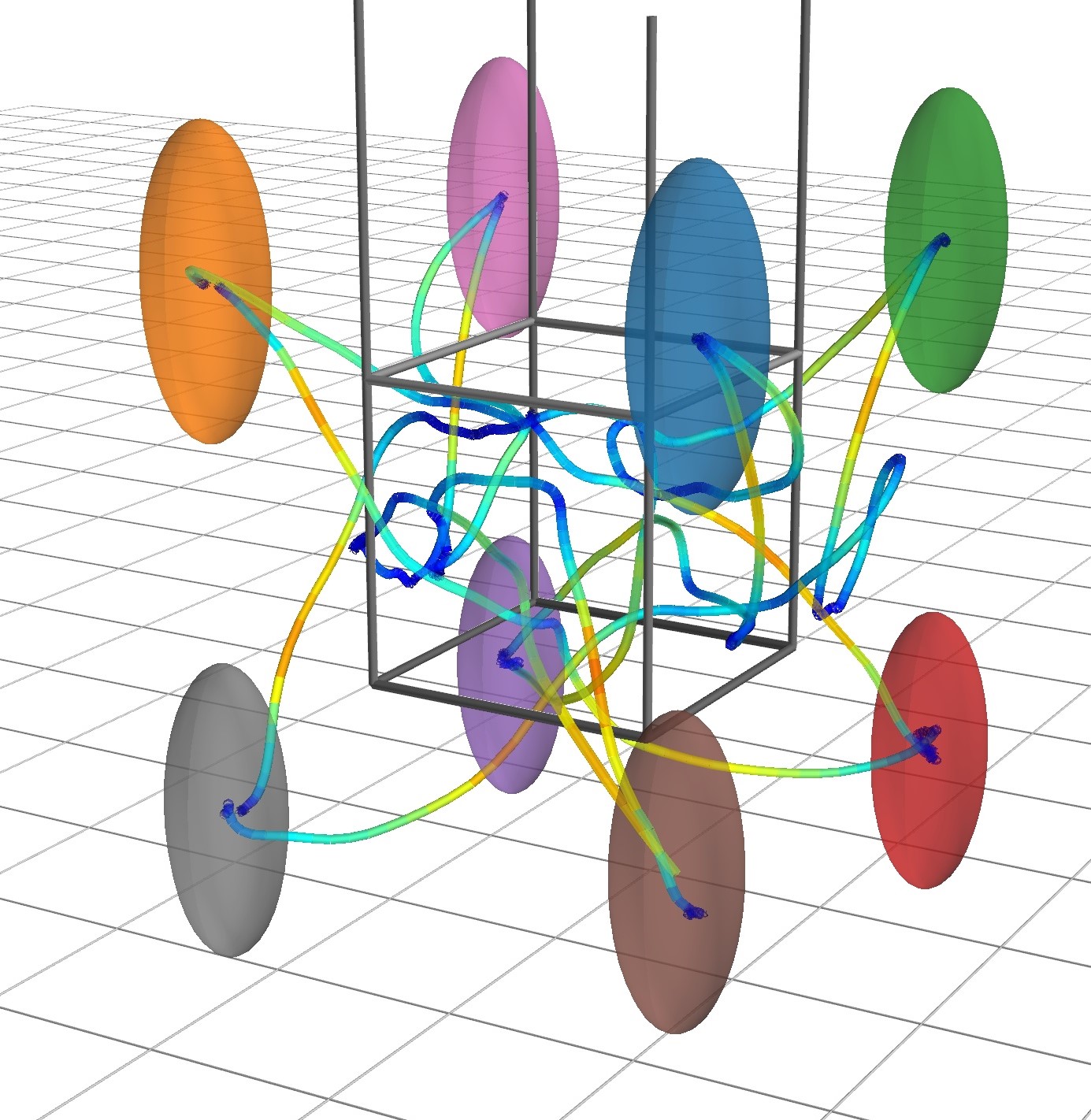}
     	\includegraphics[width=0.43\linewidth,height=0.40\linewidth]{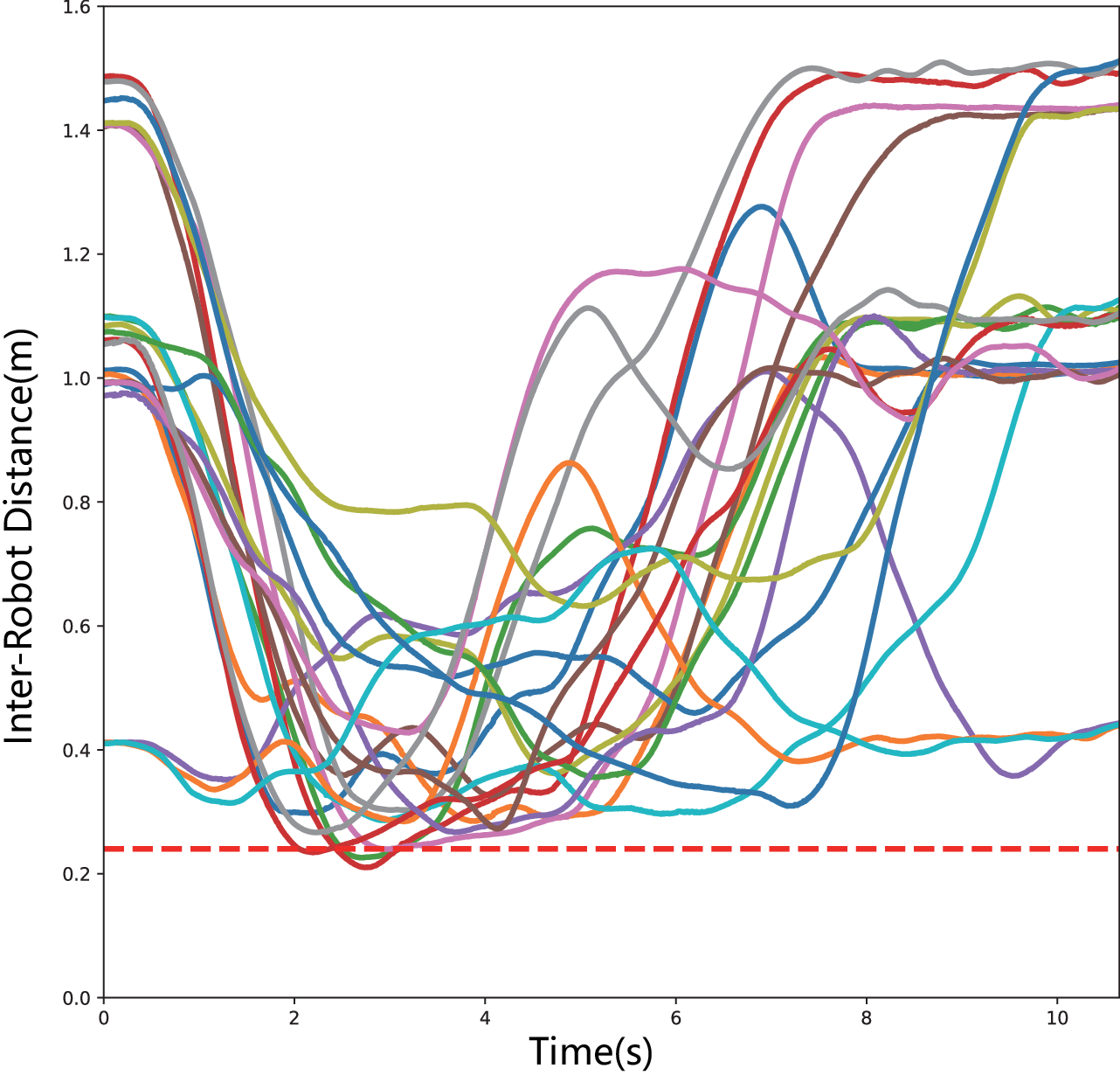}
		\label{cubic-framework}
	}
	\subfigure[\textbf{Left}: four nano-quadrotors are deployed in a polygon-shape environment. \textbf{Right}: Zoomed-in snapshot when one quadrotor navigates through the narrow passage, while another quadrotor makes room via the proposed deadlock resolution scheme.]{
		\includegraphics[width=0.42\linewidth,height=0.4\linewidth]{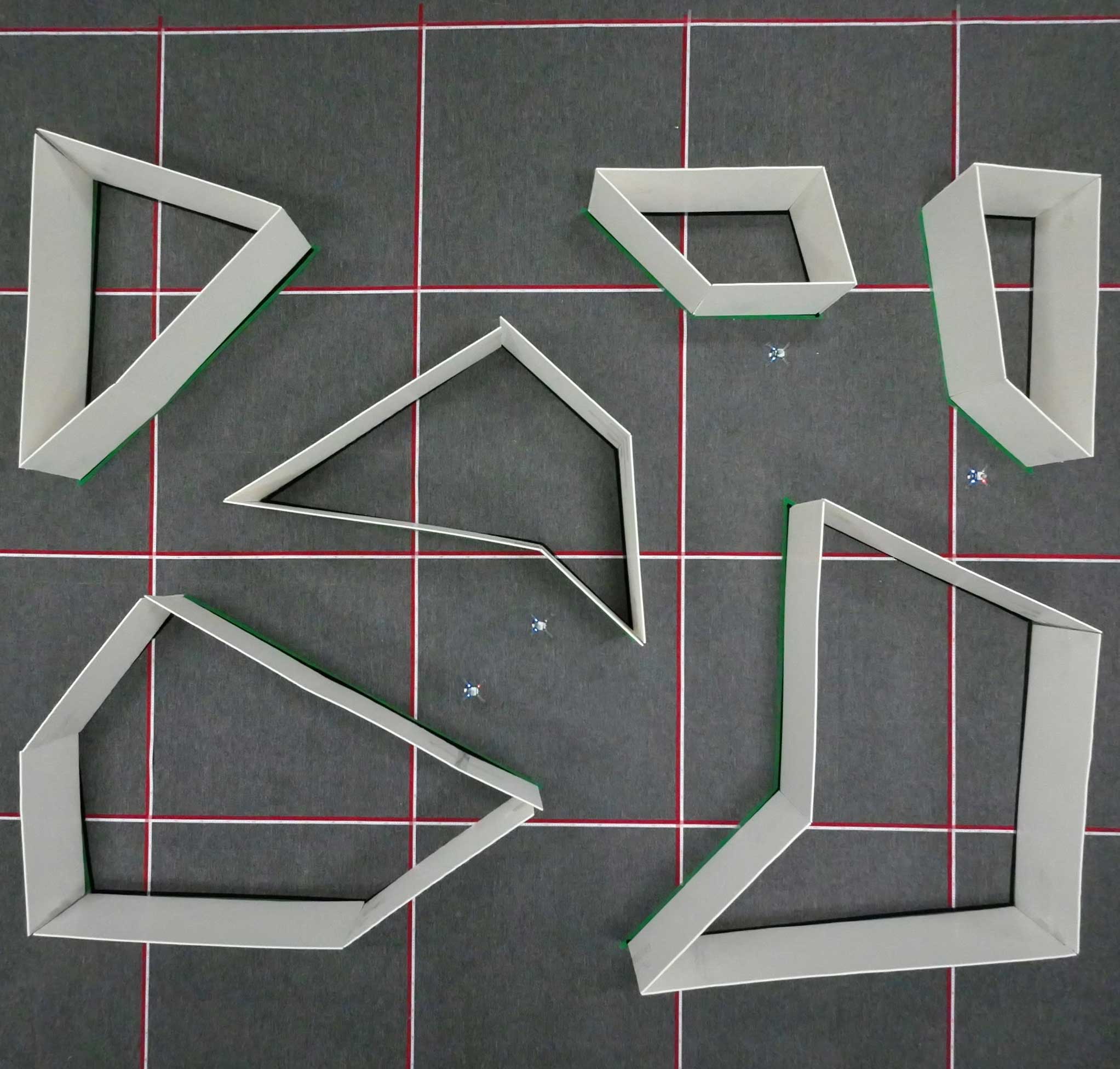}
		\includegraphics[width=0.42\linewidth,height=0.4\linewidth]{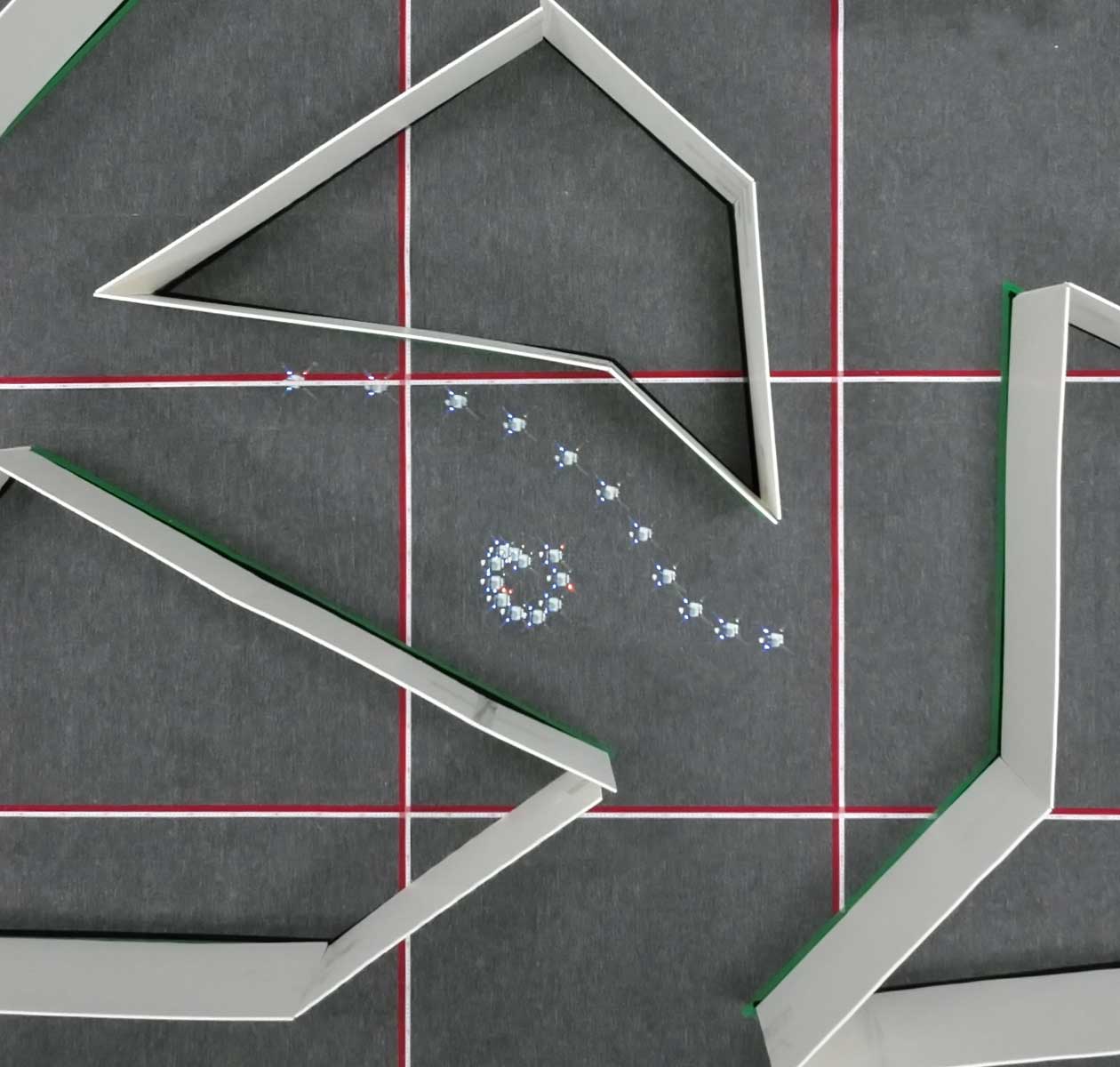}
		\label{complex-scenario}
	}
	\subfigure[\textbf{Left}: experiment results of 8 robots navigating in the ``H" scenario. \textbf{Right}: experiment results of 6 robots switching positions in the ``n" scenario. ]{
		\includegraphics[width=0.42\linewidth,height=0.4\linewidth]{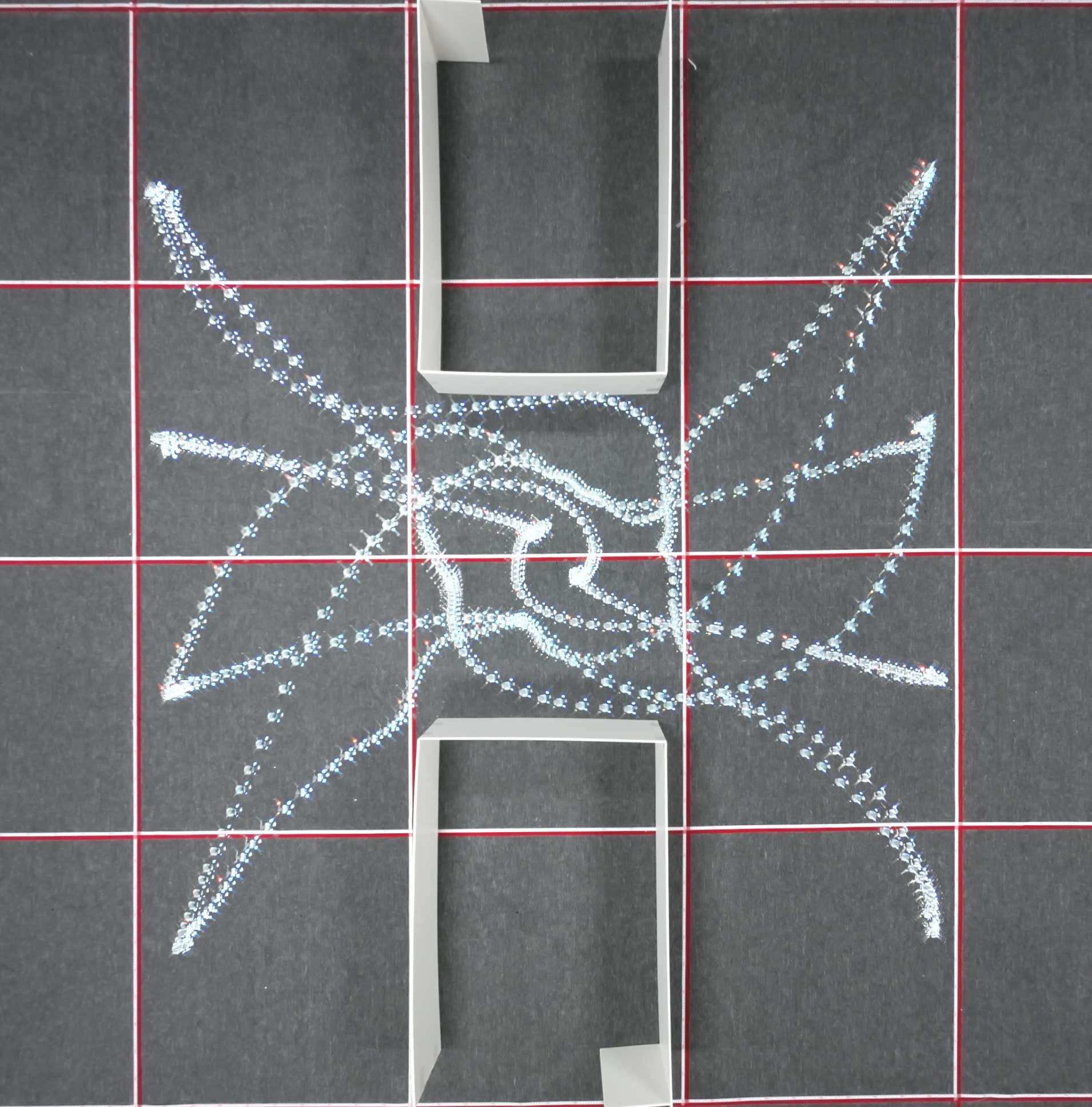}
		\includegraphics[width=0.42\linewidth,height=0.4\linewidth]{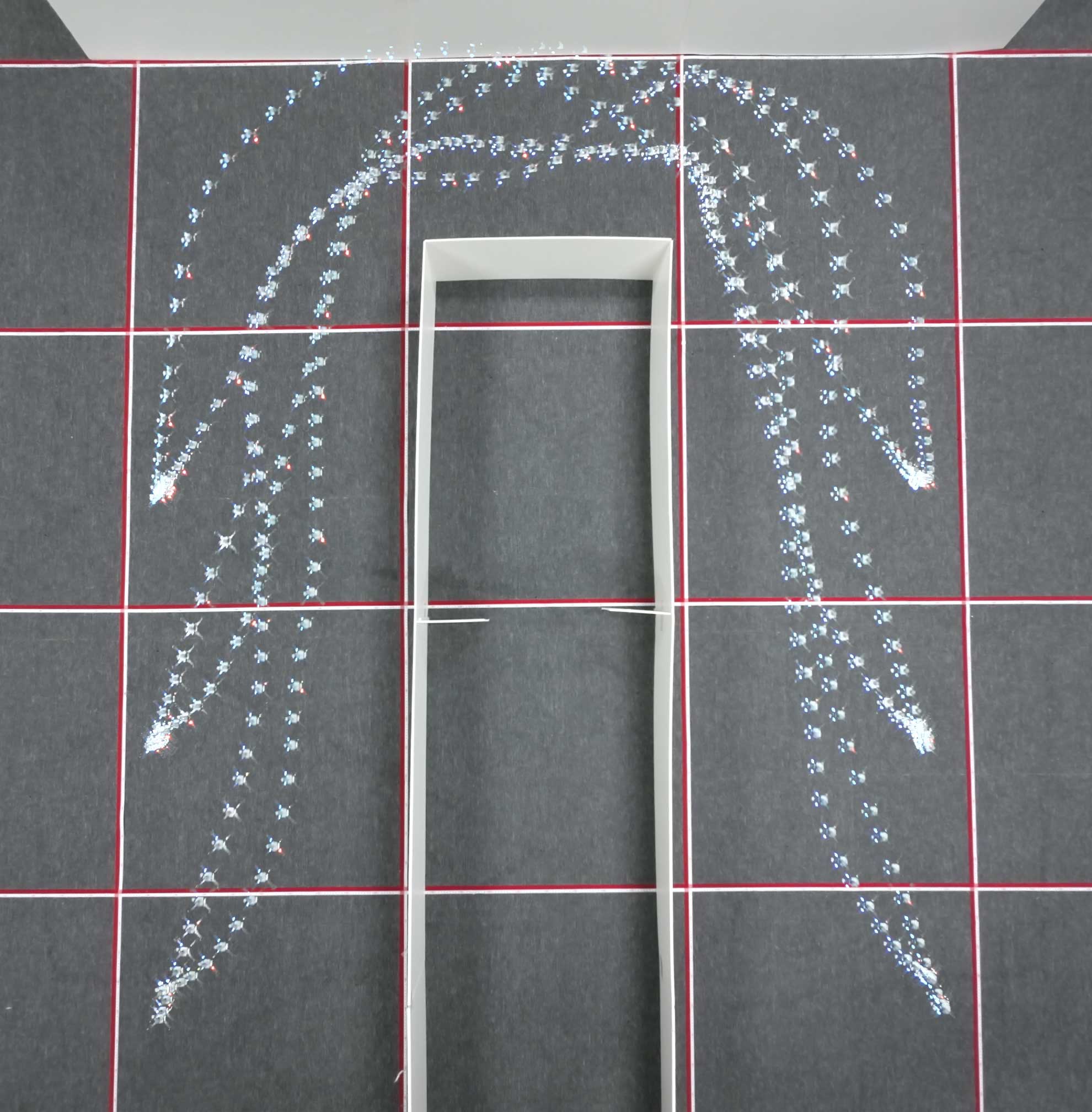}
		\label{other-exp}
	}
	\caption{Different test environments in the hardware experiments.}
	\label{experiment}
\end{figure}

\subsection{Hardware Experiments}

Hardware experiments are executed on the platform of crazyswarm \cite{crazyswarm}, where multiple nano-quadrotors are tracked by the motion capture system OptiTrack.
The computation of all robots takes place on a central computer with multiple processes at $5$Hz frequency to comply with the sampling time step $h=0.2$s.
For each crazyfile, a feedback controller~\cite{Mellinger2011} is adopted to track the planned trajectory.

The first experiment is shown in Fig.~\ref{cubic-framework}, where $8$ crazyfiles fly through a $0.6{\rm m}$ cubic framework.
Considering the air turbulence, a crazyfile in the inter-robot avoidance is represented as an ellipsoid with diameter $0.24$m in the $x-y$ plane and $0.6$m in the $z$ axis.
Owing to the deformation of robots, the inter-robot constraints are accordingly adjusted by modifying $a^{i j}_k$ and $b^{i j}_k$ as
\begin{equation*}
	a_{k}^{i j}=E \frac{ E (\overline{p}_{k}^{i}- \overline{p}_{k}^{j})  } { \|E (\overline{p}_{k}^{i}-\overline{p}_{k}^{j}) \|_{2} },\quad
	b_{k}^{i j}={a_{k}^{i j}}^\mathrm{T} \frac{E (\overline{p}_{k}^{i} +  \overline{p}_{k}^{j})}{2}+\frac{r_{\min }}{2},
\end{equation*}
where $E={\rm diag}(1.0,1.0,0.4)$, $r_{\min }=\sqrt{4 {r_a}^2+v_{\text{max}}^2}$ and $r_a=0.12$m.
From the result given in Fig~\ref{cubic-framework}, it is apparent that the crazyfiles can achieve this transition.

In addition, four crazyfiles moving in a polygonal environment as shown in Fig~\ref{complex-scenario}.
Given initial positions, the targets are randomly chosen.
After arriving at the targets, the new one are published immediately and this process is repeated $5$ times.
In this scenario, the feasible space is the irregular-shaped passage at the interval of polygon-shaped obstacles where the width of these passages ranges from~$0.4$m to~$0.7$m.
Via the proposed deadlock resolution scheme, particularly the dynamic-priority mechanism, a robot can smoothly squeeze out a way between the obstacle and another lower-priority robot as shown in Fig.~\ref{complex-scenario}.

At last, other two experiments as shown in Fig.~\ref{other-exp} are carried out,
where the scenarios are ``H" and ``n" environments respectively.
Therein, both the safety and deadlock resolution are achieved in narrow passages.
In ``H"-transition, the potential deadlock is resolved via right-hand rule and the transition is completed within $9$s.
Regarding the ``n"-transition, the intersection of two groups of quadrotors at the top passage is the main challenge for this mission.
It can be seen that the proposed method resolves it as the quadrotors fly through the passage without sacrificing any speed.
As a result, the task is finished within $8$s.

\section{Conclusion}

This work has proposed a novel multi-robot trajectory planning method for obstacle-dense environments, where the collision avoidance is guaranteed and
potential deadlocks are resolved online.
Comprehensive simulations and hardware experiments have been conducted to validate the proposed method.
However, a theoretical guarantee on the resolution of deadlocks or even livelocks remains open,
which is part of our ongoing work.

\bibliographystyle{IEEEtran}
\bibliography{REF}

\end{document}